\documentclass{article}
\usepackage[dvips]{graphicx}
\usepackage{amssymb,amsmath,color}
\usepackage{url}

\newcommand\eqdef{\ensuremath{\stackrel{\rm def}{=}}} 

\title{Shaping Level Sets with Submodular Functions}

\author{
Francis Bach  \\
INRIA - Sierra project-team\\
Laboratoire d'Informatique de l'Ecole Normale Sup\'erieure \\
Paris, France \\
\texttt{francis.bach@ens.fr} }

\newcommand{\BEAS}{\begin{eqnarray*}}
\newcommand{\EEAS}{\end{eqnarray*}}
\newcommand{\BEA}{\begin{eqnarray}}
\newcommand{\EEA}{\end{eqnarray}}
\newcommand{\BEQ}{\begin{equation}}
\newcommand{\EEQ}{\end{equation}}
\newcommand{\BIT}{\begin{itemize}}
\newcommand{\EIT}{\end{itemize}}
\newcommand{\BNUM}{\begin{enumerate}}
\newcommand{\ENUM}{\end{enumerate}}
\newcommand{\BA}{\begin{array}}
\newcommand{\EA}{\end{array}}

\newcommand{\idm}{I}
\newcommand{\rb}{\mathbb{R}}
\newcommand{\BlackBox}{\rule{1.5ex}{1.5ex}}  
\newcommand{\lova}{Lov\'asz }

\newenvironment{proof}{\par\noindent{\bf Proof\ }}{\hfill\BlackBox\\[2mm]}
\newtheorem{lemma}{Lemma}
\newtheorem{theorem}{Theorem}
\newtheorem{proposition}{Proposition}

\newcommand{\mysec}[1]{Section~\ref{sec:#1}}
\newcommand{\eq}[1]{Eq.~(\ref{eq:#1})}
\newcommand{\myfig}[1]{Figure~\ref{fig:#1}}

\begin{document}

\maketitle

\vspace*{-.2cm}

\begin{abstract}

\vspace*{-.0cm}

We consider a class of sparsity-inducing regularization terms based on submodular functions. While previous work has focused on
non-decreasing functions, we explore symmetric submodular functions and their \lova extensions. We show that the \lova extension may be seen as the convex envelope of a function that depends on level sets  (i.e., the set of indices whose corresponding components of the underlying predictor are greater than a given constant): this leads to a class of convex structured regularization terms that impose prior knowledge on the level sets, and not only on the supports of the underlying predictors. We provide a unified set of optimization algorithms, such as proximal operators, and theoretical guarantees (allowed level sets and recovery conditions). By selecting specific submodular functions,  we   give a new interpretation to known norms, such as the total variation; we also define new norms, in particular ones that are based on order statistics with application to clustering and outlier detection, and on noisy cuts in graphs with application to change point detection in the presence of outliers.

\vspace*{-.1cm}

\end{abstract}

\section{Introduction}

\vspace*{-.15cm}

The concept of parsimony is central in many scientific domains. In the context of statistics, signal processing or machine learning, it may take several forms. Classically, in a variable or feature selection problem, a sparse solution with many zeros is sought so that the model is either more interpretable, cheaper to use, or simply matches available prior knowledge (see, e.g.,~\cite{Zhaoyu,negahban2009unified,bach2010structured} and references therein). 
In this paper, we instead consider sparsity-inducing regularization terms that will lead to  solutions with \emph{many equal values}. A classical example is the total variation in one or two dimensions, which leads to piecewise constant solutions~\cite{tibshirani2005sparsity,chambolle2009total} and can be applied to various image labelling problems~\cite{boykov2001fast,chambolle2009total}, or change point detection tasks~\cite{harchaoui2008catching,bleakley,kolar2009sparsistent}. Another example is the ``Oscar'' penalty which induces automatic grouping of the features~\cite{oscar}. In this paper, we follow the approach of~\cite{bach2010structured}, who designed sparsity-inducing norms based on \emph{non-decreasing} submodular functions, as a convex approximation to imposing a specific prior on the \emph{supports} of the predictors. Here, we show that a similar parallel holds for some other class of submodular functions, namely non-negative set-functions which are equal to zero for the full and empty set. Our main instance of such functions are \emph{symmetric} submodular functions.

We make the following contributions:
\BIT    
\item[$-$] We provide in \mysec{properties} explicit links between priors on level sets and certain submodular functions: we show  that the \lova extensions (see, e.g.,~\cite{submodular_tutorial} and a short review in \mysec{submodular}) associated to  these  submodular functions are the convex envelopes (i.e., tightest convex lower bounds) of  specific functions that depend on all level sets of the underlying vector.

\item[$-$] In \mysec{examples}, we reinterpret existing norms such as the total variation and design new norms, based on noisy cuts or order statistics. We propose applications to clustering and outlier detection, as well as to change point detection in the presence of outliers.

\item[$-$] We provide unified algorithms in \mysec{algorithms}, such as proximal operators, which are based on a sequence of submodular function minimizations (SFMs), when such SFMs are efficient, or by adapting the generic slower approach of~\cite{bach2010structured} otherwise.

\item[$-$] We derive unified theoretical guarantees for level set recovery in \mysec{theory}, showing that even in the absence of correlation between predictors, level set recovery is not always guaranteed, a situation which is to be contrasted with traditional support recovery situations~\cite{Zhaoyu,bach2010structured}.
 \EIT

\vspace*{-.2cm}

\paragraph{Notation.}
For $w \in \rb^p$ and $q \in [1,\infty]$, we  denote by $\| w\|_q$ the $\ell_q$-norm
of $w$.  Given a subset $A $ of  $V = \{1,\dots,p\}$, $1_A \in \{0,1\}^p$ is the indicator vector of the subset $A$. 
Moreover, given a vector $w$ and a matrix~$Q$, $w_A$ and $Q_{AA}$ denote the corresponding subvector and submatrix of $w$ and $Q$. Finally, for $w \in \rb^p$ and $A \subset V$, $w(A)=\sum_{k \in A} w_k = w^\top 1_A $ (this defines a modular set-function). In this paper, for a certain vector $w \in \rb^p$, we call \emph{level sets} the sets of indices which are larger (or smaller) or equal to a certain constant $\alpha$, which we denote $\{ w \geqslant \alpha\}$
(or $\{ w \leqslant \alpha\}$), while we call \emph{constant sets} the sets of indices which are equal to a constant $\alpha$, which we denote $\{ w = \alpha\}$.

\section{Review of Submodular Analysis}
\label{sec:submodular}

\vspace*{-.1cm}

In this section, we review relevant results from submodular analysis. For more details, see, e.g.,~\cite{fujishige2005submodular}, and, for a review with proofs derived from classical convex analysis, see, e.g., \cite{submodular_tutorial}.

\textbf{Definition.} \hspace*{.05cm}
Throughout this paper, we consider a \emph{submodular} function $F$ defined on the power set $2^V$ of $V = \{1,\dots,p\}$, i.e., such that $
 \forall A,B \subset V, \   F(A) + F(B) \geqslant F(A \cup B) + F(A \cap B)$.
 Unless otherwise stated, we consider functions which are non-negative (i.e., such that $F(A) \geqslant 0$ for all $A \subset V$), and that satisfy $F(\varnothing) = F(V)=0$. Usual examples are symmetric submodular functions, i.e., such that $\forall A \subset V, F(V \backslash A) = F(A)$, which are known to always have non-negative values. We give several examples in \mysec{examples}; for illustrating the   concepts introduced in this section  and \mysec{properties}, we will consider the cut in an undirected chain graph, i.e., $F(A) = \sum_{j=1}^{p-1}|(1_A)_{j} - (1_A)_{j+1}|$.
   
\textbf{\lova extension.} \hspace*{.05cm}
 Given any set-function $F$ such that $F(V)=F(\varnothing)=0$, one can define its \emph{\lova extension} $f: \rb^p \to \rb$, as $
 f(w) = \int_\rb F( \{w \geqslant \alpha \}) d\alpha$ (see, e.g., \cite{submodular_tutorial} for this particular formulation).
The \lova extension is convex if and only if $F$ is submodular. Moreover, $f$ is piecewise-linear and for all $A \subset V$, $f(1_A) = F(A)$, that is, it is indeed an extension from $2^V$ (which can be identified to $\{0,1\}^p$ through indicator vectors) to $\rb^p$. Finally, it is always positively homogeneous. For the chain graph, we obtain the usual total variation $f(w) =\sum_{j=1}^{p-1} |w_j-w_{j+1}|$.

\textbf{Base polyhedron.} \hspace*{.05cm}
We denote by $B(F)= \{ s \in \rb^p, \ \forall A \subset V, \  s(A)\leqslant F(A), \ s(V) = F(V) \}$ the \emph{base polyhedron}~\cite{fujishige2005submodular}, 
where we use the notation $s(A) = \sum_{k \in A} s_k$.
One important result in submodular analysis is that if $F$ is a submodular function, then we have a representation of $f$ as a maximum of linear functions~\cite{fujishige2005submodular,submodular_tutorial}, i.e.,  for all $w \in \rb^p$, $
f(w) = \max_{ s \in B(F) } \  w^\top s$. 
Moreover, instead of solving a linear program with $2^p$ contraints, a solution $s$ may   be obtained by the following ``greedy algorithm'':
order the components of $w$ in decreasing order $w_{j_1} \geqslant \dots \geqslant w_{j_p}$, and then take for all $k \in   \{1,\dots,p\}$,
$s_{j_k} = F( \{ j_1,\dots,j_k\} ) - F( \{ j_1,\dots,j_{k-1}\} ) .$ 

\textbf{Tight and inseparable sets.} \hspace*{.05cm} The polyhedra~$\mathcal{U} = \{ w \in \rb^p, f(w) \leqslant 1 \}$ and $B(F)$ are polar to each other~(see, e.g.,~\cite{rockafellar97} for definitions and properties of polar sets). Therefore, the facial structure of $ \mathcal{U}$ may be obtained from the one of $B(F)$. Given $s \in B(F)$, a set $A \subset V$ is said \emph{tight} if $s(A)=F(A)$. It is known that the set of tight sets is a distributive lattice, i.e., if $A$ and $B$ are tight, then so are $A \cup B$ and $A \cap B$~\cite{fujishige2005submodular,submodular_tutorial}. 
The faces of $B(F)$ are thus intersections of hyperplanes $\{ s(A) = F(A) \}$ for $A$  belonging to certain distributive lattices (see Prop.~\ref{prop:faces}).
A set $A$ is said \emph{separable} if there exists a non-trivial partition of $A = B \cup C$ such that $F(A) = F(B)+F(C)$. A set is said inseparable if it is not separable. For the cut in an undirected graph, inseparable sets are exactly connected sets.

\section{Properties of the \lova Extension}
\label{sec:properties}
 
\vspace*{-.1cm}

 In this section, we derive properties of the \lova extension for submodular functions, which go beyond convexity and homogeneity. Throughout this section, we assume that $F$ is  a non-negative submodular  set-function that is equal to zero at $\varnothing$ and $V$. This immediately implies that  $f$   is invariant by addition of any constant vector (that is, $f(w+ \alpha 1_V) = f(w)$ for all $w \in \rb^p$ and $\alpha \in \rb$), and that $f(1_V) = F(V) = 0$. Thus, contrary to the non-decreasing case~\cite{bach2010structured}, our regularizers are not norms. However, they are norms on the hyperplane $\{w^\top 1_V = 0\}$ as soon as for $A \neq \varnothing$ and $A \neq V$, $F(A)>0$ (which we assume for the rest of this paper).
 
 We now show that the \lova extension is the convex envelope of a certain combinatorial function which does depend on all levets sets $\{ w \geqslant \alpha\}$ of $w \in \rb^p$  (see proof in supplementary material):
\begin{proposition}[Convex envelope]
\label{prop:envelope}

The \lova extension $f(w)$ is the convex envelope of the function
$w \mapsto \max_{ \alpha \in \rb}  F( \{ w \geqslant \alpha \} )$ on the set $[0,1]^p + \rb 1_V = \{ w \in \rb^p , \ \max_{k \in V} w_k - \min_{k \in V} w_k \leqslant 1 \}$.
\end{proposition}

Note the difference with the result of~\cite{bach2010structured}: we consider here a different set on which we compute the convex envelope ($[0,1]^p + \rb 1_V$ instead of $[-1,1]^p$), and not a function of the support of $w$, but of \emph{all} its level sets.\footnote{Note that the support $\{ w = 0 \}$ is a constant set which is the intersection of two level sets.} Moreover, the \lova extension is a convex relaxation of a function of \emph{level sets} (of the form $\{w \geqslant \alpha\}$) and not of \emph{constant sets} (of the form $\{w = \alpha\}$). It
 would have been perhaps more intuitive to consider for example $\int_\rb F( \{ w = \alpha \} ) d\alpha$, since it does not depend on the ordering of the values that $w$ may take; however, the latter function does not lead to a convex function amenable to polynomial-time algorithms.
 This definition through level sets will generate some potentially undesired behavior (such as the well-known staircase effect for the one-dimensional total variation), as we show in \mysec{theory}.

The next proposition describes the set of extreme points of  the ``unit ball'' $\mathcal{U} = \{ w, \ f(w) \leqslant 1\}$, giving a first illustration of sparsity-inducing effects (see example in \myfig{balls}). 
\begin{proposition}[Extreme points]
\label{prop:extreme}
The extreme points of the set  $\mathcal{U} \cap \{ w^\top 1_V =0 \}$ are the projections of the vectors $1_A / F(A)$ on the plane $\{  w^\top 1_V = 0\}$, for $A$ such that $A$ is inseparable for $F$ and $V \backslash A$ is inseparable for $B \mapsto F(A \cup B) - F(A)$.
\end{proposition}

\begin{figure}
\begin{center}
 
 \vspace*{-.75cm}
 
  \hspace*{-.5cm} \includegraphics[scale=.42]{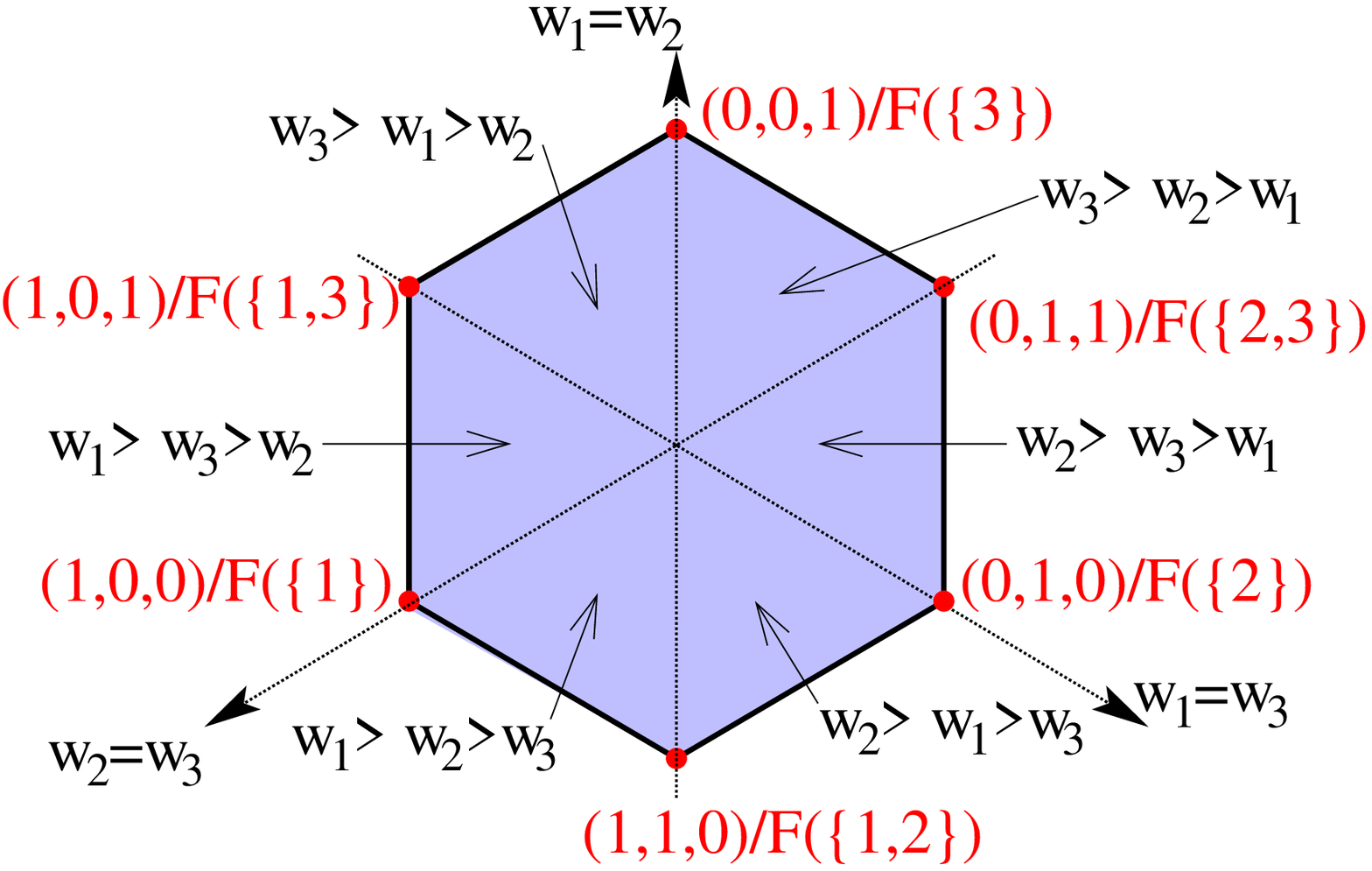}  \hspace*{-.5cm}
\includegraphics[scale=.42]{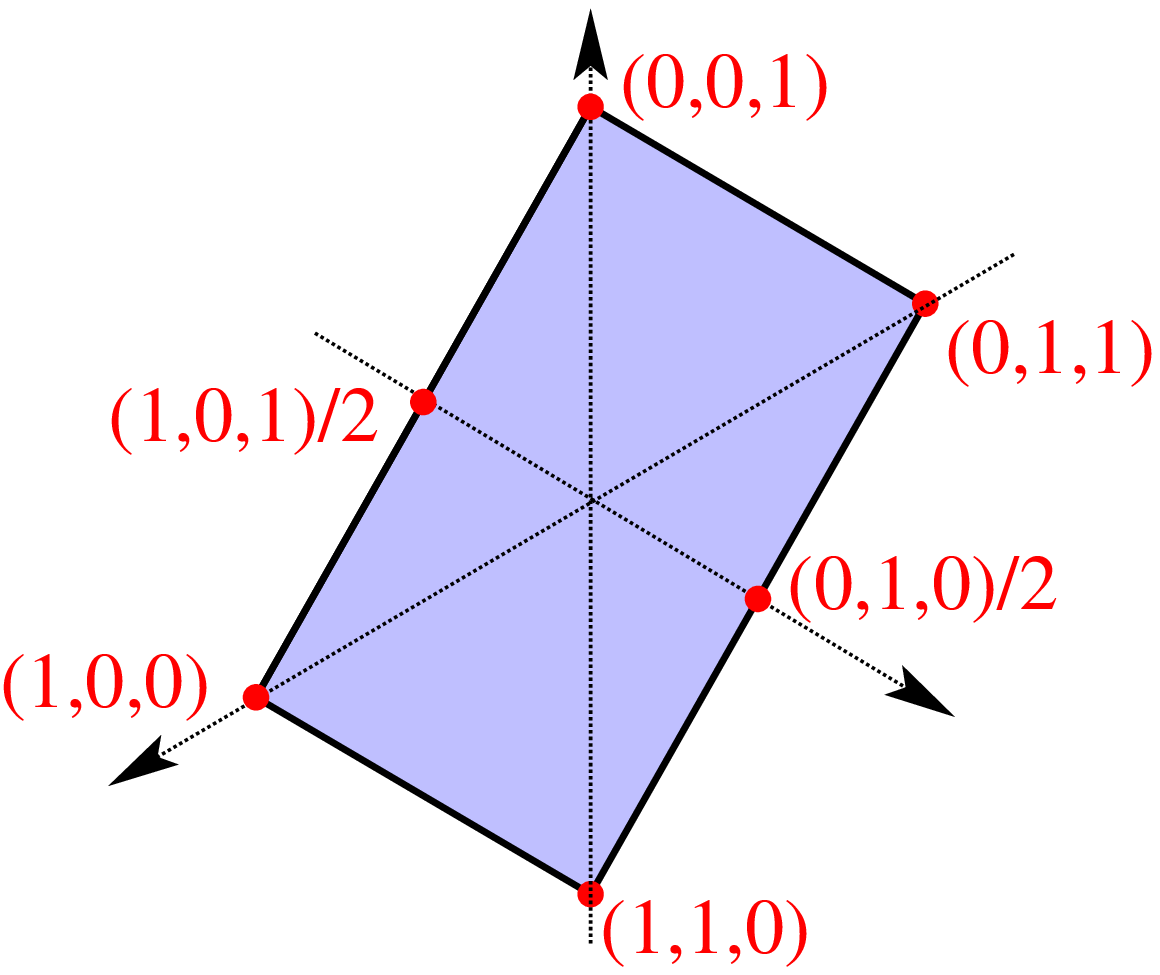}   \hspace*{-.5cm}
\end{center}

 \vspace*{-.6cm}

\caption{Top: Polyhedral level set of $f$ (projected on the set $ w^\top 1_V=0$), for 2 different submodular symmetric functions of three variables, with different inseparable sets leading to different sets of extreme points; changing values of $F$ may make some of the extreme points disappear. The various extreme points cut the space into polygons where the ordering of the component is fixed. Left: $F(A) = 1_{|A| \in \{1,2\}}$ (all possible extreme points); note that the polygon need not be symmetric in general. Right: one-dimensional total variation on three nodes, i.e., $F(A) = |1_{ 1 \in A } - 1_{ 2 \in A }| + |1_{ 2 \in A } - 1_{ 3 \in A }|$, leading to $f(w) = |w_1-w_2| + |w_2-w_3|$, for which the extreme points corresponding to the separable set $\{1,3\}$ and its complement disappear.
}
\label{fig:balls}

  \vspace*{-.25cm}

\end{figure}

\textbf{Partially ordered sets and distributive lattices.} \hspace*{.05cm}
A subset $\mathcal{D}$ of $2^V$ is a (distributive) lattice if it is invariant by intersection and union. We assume in this paper that all lattices contain the empty set $\varnothing$ and the full set $V$, and we endow the lattice with the inclusion order.
Such lattices may be represented as a \emph{partially ordered set (poset)} $\Pi(\mathcal{D}) = \{ A_1,\dots,A_m\}$ (with order relationship $\succcurlyeq$), where the sets $A_j$, $j=1,\dots,m$, form a \emph{partition} of $V$ (we always assume a topological ordering of the sets, i.e., $A_i \succcurlyeq A_j \Rightarrow i \geqslant  j$). As illustrated in \myfig{posets}, we go from $\mathcal{D}$ to $\Pi(\mathcal{D})$, by considering all maximal chains in $\mathcal{D}$ and the differences between consecutive sets. 
We go from $\Pi(\mathcal{D})$ to $\mathcal{D}$, by constructing all \emph{ideals} of $\Pi(\mathcal{D})$, i.e., sets $J$ such that if an element of $\Pi(\mathcal{D})$ is lower than an element of $J$, then it has to be in $J$~(see \cite{fujishige2005submodular} for more details, and an example in \myfig{posets}).
Distributive lattices and posets are thus in one-to-one correspondence. Throughout this section, we go back and forth between these two representations. The distributive lattice will correspond to all authorized level sets $\{ w \geqslant \alpha\}$ in a single face of $\mathcal{U}$, while the elements of the poset are the constant sets (over which $w$ is constant), with the order between the subsets giving \emph{partial} constraints between the values of the corresponding constants.

\begin{figure}
\begin{center}
 
 \vspace*{-.25cm}
 
\includegraphics[scale=.6]{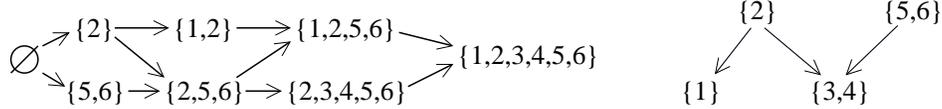}

\vspace*{-.5cm}

\end{center}
\caption{Left: distributive lattice with 7 elements in $2^{\{1,2,3,4,5,6\}}$, represented with the Hasse diagram corresponding to the inclusion order (for a partial order, a Hasse diagram connects $A$ to $B$ if $A$ is smaller than $B$ and there is no $C$ such that $A$ is smaller than $C$ and $C$ is smaller than $B$). Right: corresponding poset, with 4 elements that form a partition of $\{1,2,3,4,5,6\}$, represented with the Hasse diagram corresponding to the order $\succcurlyeq$ (a node points to its immediate smaller node according to $\succcurlyeq$). Note that this corresponds to an ``allowed''  lattice  (see Prop.~\ref{prop:faces}) for the one-dimensional total variation.
}
\label{fig:posets}

 \vspace*{-.5cm}

\end{figure}

\textbf{Faces of $\mathcal{U}$.}  \hspace*{.05cm}
The faces of $\mathcal{U}$ are characterized by lattices $\mathcal{D}$, with their corresponding posets $\Pi(\mathcal{D}) = \{ A_1,\dots,A_m\}$. We denote by $\mathcal{U}^\circ_\mathcal{D}$ (and by $\mathcal{U}_\mathcal{D}$ its closure) the set of $w \in \rb^p$ such that (a) $w$ is piecewise constant with respect to $\Pi(\mathcal{D})$, with value $v_i$ on $A_i$, and (b) for all pairs $(i,j)$, $A_i \succcurlyeq  A_j \Rightarrow v_i > v_j$. For certain lattices $\mathcal{D}$, these will be exactly the relative interiors of all faces of~$\mathcal{U}$:
\begin{proposition}[Faces of $\mathcal{U}$] 
\label{prop:faces}
The (non-empty) relative interiors of all faces of $\mathcal{U}$ are exactly of the form $\mathcal{U}^\circ_\mathcal{D}$, where $\mathcal{D}$ is a lattice such that:
\\[.1cm]
  (i) the restriction of $F$  to $\mathcal{D}$ is modular, i.e., for all $A,B \in \mathcal{D}$,
$F(A) + F(B) = F(A \cup B) + F( A \cap B)$,
\\[.1cm]
(ii) for all  $j \in \{1,\dots,m\}$, the set $A_j$ is inseparable for the function $C_j \mapsto F(B_{j-1} \cup C_j) - F(B_{j-1})$, where $B_{j-1}$ is the union of all ancestors of $A_j$ in $\Pi(\mathcal{D})$,
\\[.1cm]
(iii) among all lattices corresponding to the same \emph{unordered} partition, $\mathcal{D}$ is a maximal element of the set of lattices satisfying (i) and (ii).
\end{proposition}
Among the three conditions, the second one is the easiest to interpret, as it reduces to having constant sets which are inseparable for certain submodular functions, and for cuts in an undirected graph, these will exactly be connected sets.

Since we are able to characterize \emph{all} faces of $\mathcal{U}$ (of all dimensions) with non-empty relative interior, we have a partition of the space and any $w \in \rb^p$ which is not proportional to $1_V$, will be, up to the strictly positive constant $f(w)$, in exactly one of these relative interiors of faces; we refer to this lattice as the \emph{lattice associated to} $w$. Note that from the face $w$ belongs to, we have strong constraints on the constant sets, but we may not be able to determine all level sets of $w$, because only partial constraints are given by the order on $\Pi(\mathcal{D})$. For example, in \myfig{posets}, $w_2$ may be larger or smaller than $w_5 = w_6$ (and even potentially equal, but with zero probability, see \mysec{theory}).

\section{Examples of Submodular Functions}
\label{sec:examples}

\vspace*{-.1cm}

In this section, we provide examples of submodular functions and of their \lova extensions. Some are well-known (such as cut functions and total variations), some are new in the context of supervised learning (regular functions), while some have interesting effects in terms of clustering or outlier detection (cardinality-based functions).

\textbf{Symmetrization.} \hspace*{.05cm}
From any submodular function $G$, one may define
$F(A) = G(A) + G( V \backslash A) - G(\varnothing) - G(V)$, which is symmetric. Potentially interesting examples which are beyond the scope of this paper are mutual information, or functions of eigenvalues of submatrices~\cite{bach2010structured}.

\textbf{Cut functions.} \hspace*{.05cm}
\label{sec:cuts}
 Given a set of \emph{nonnegative} weights $d:V \times V \to \rb_+$, define  the cut
$F(A) = \sum_{k \in A, j \in V \backslash A} d(k,j)$.  The \lova extension is equal to 
$f(w) = \sum_{k,j \in V} d(k,j) ( w_k - w_j )_+$ (which shows submodularity because $f$ is convex), and is often referred to as the total variation. If  the weight function $d$ is symmetric, then the submodular function is also symmetric. In this case, it can be shown that inseparable sets for  functions $A \mapsto F(A \cup B) - F(B)$ are exactly \emph{connected} sets. Hence, constant sets are connected sets, which is the usual justification behind the total variation. Note however that some configurations of connected sets are not allowed due to the other conditions in Prop.~\ref{prop:faces} (see examples in \mysec{theory}).
 In \myfig{speed} (right plot), we give an example of the usual chain graph, leading to the one-dimensional total variation~\cite{tibshirani2005sparsity,chambolle2009total}.  
Note that these functions can be extended to cuts in hypergraphs, which may have interesting applications in computer vision~\cite{boykov2001fast}. Moreover, directed cuts may be interesting to favor increasing or decreasing jumps along the edges of the graph.

\begin{figure}
\begin{center}

 \vspace*{-.25cm}
 
\hspace*{-1.25cm}
\parbox[b]{10.5cm}{
\includegraphics[scale=.26]{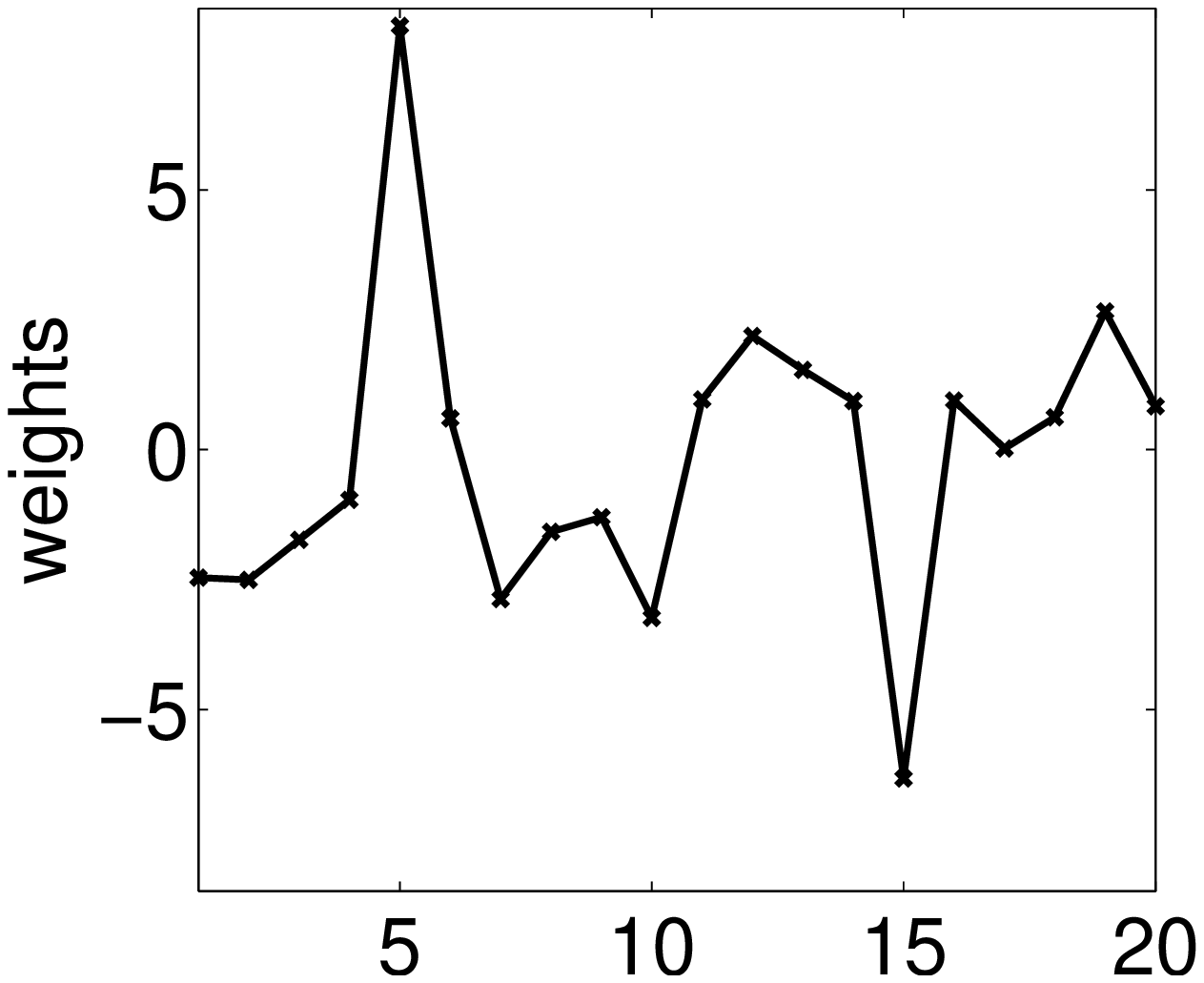} \hspace*{-.2cm}
\includegraphics[scale=.26]{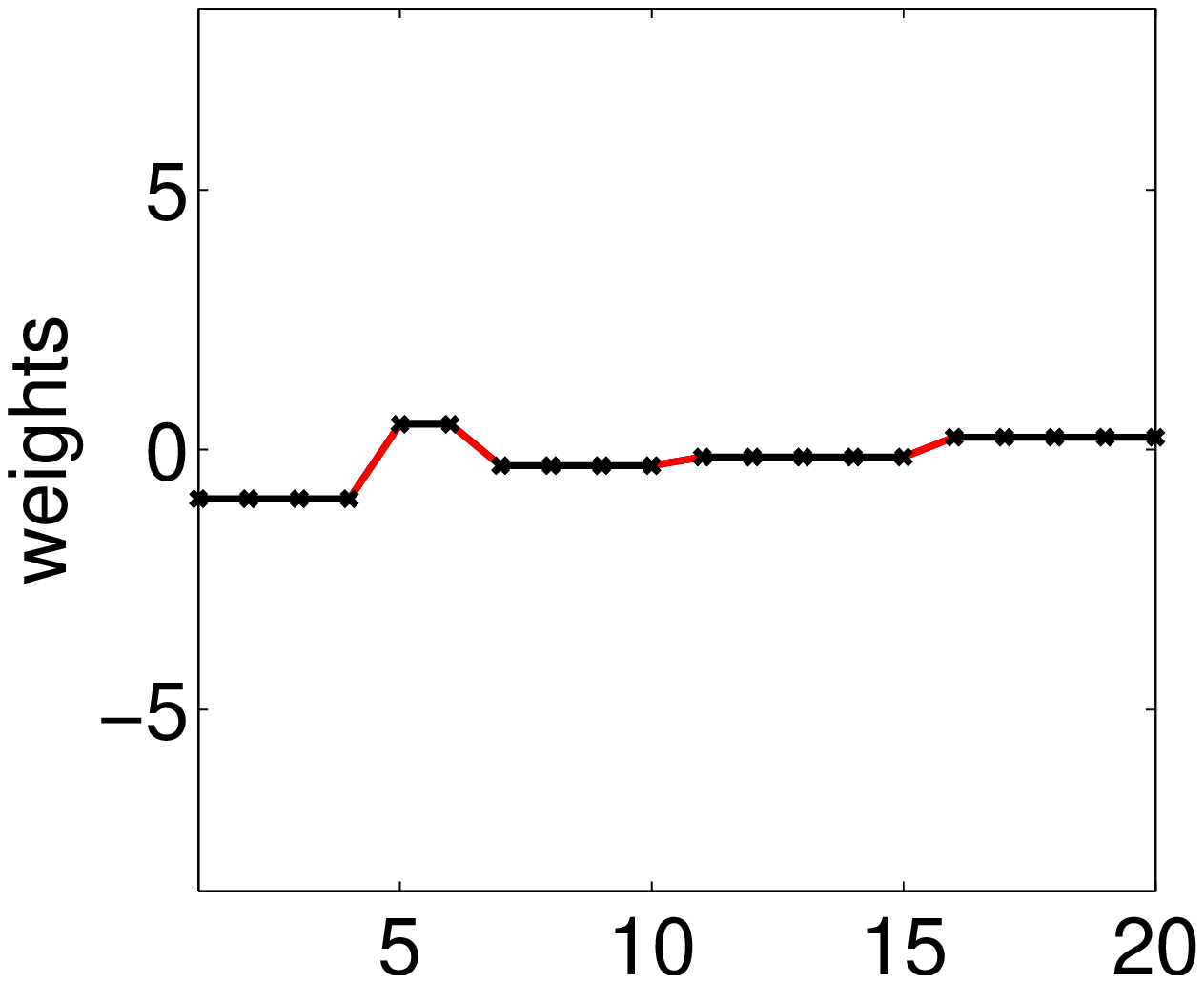} \hspace*{-.2cm}
\includegraphics[scale=.26]{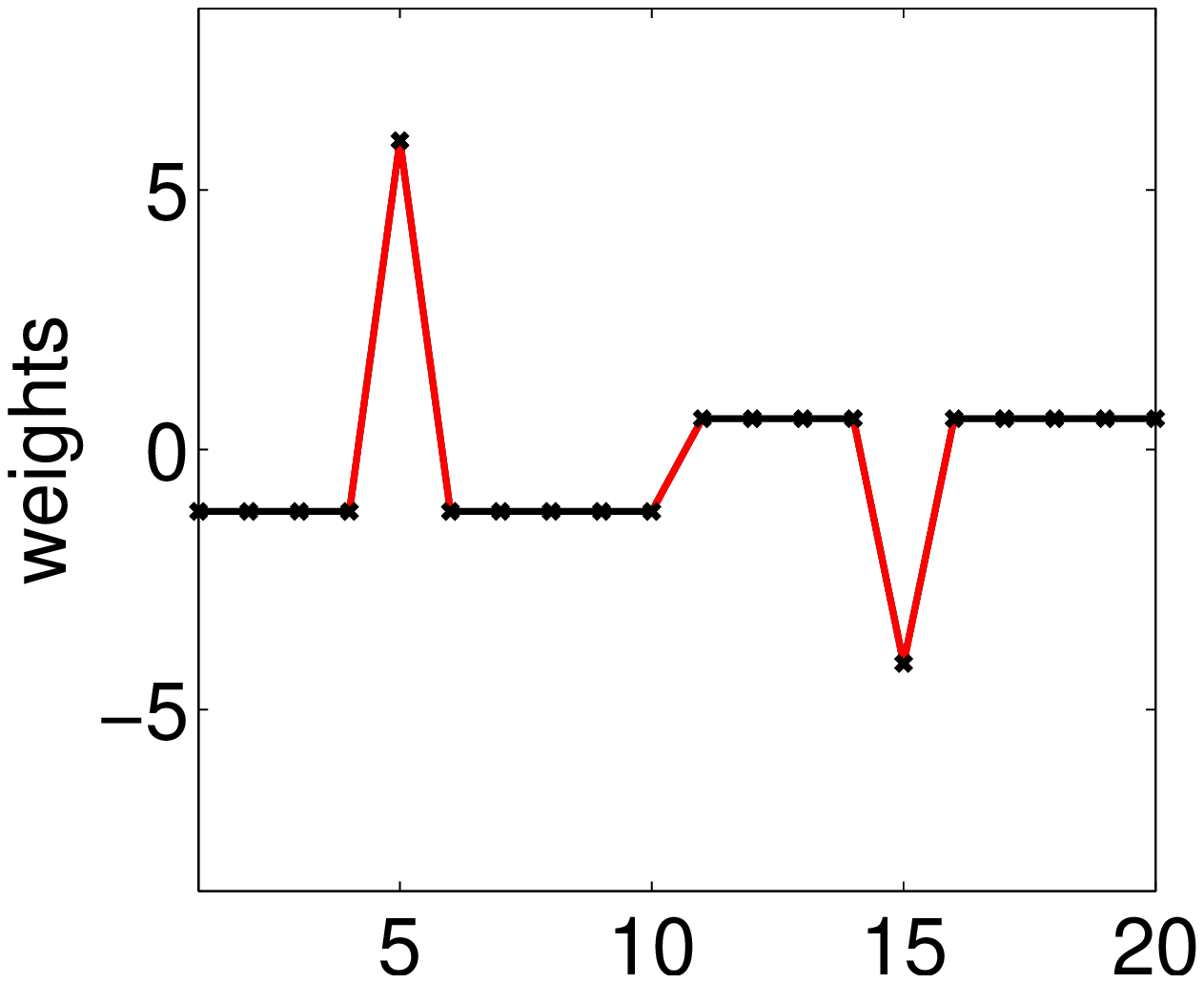}

\vspace*{.25cm}} 
\includegraphics[scale=.38]{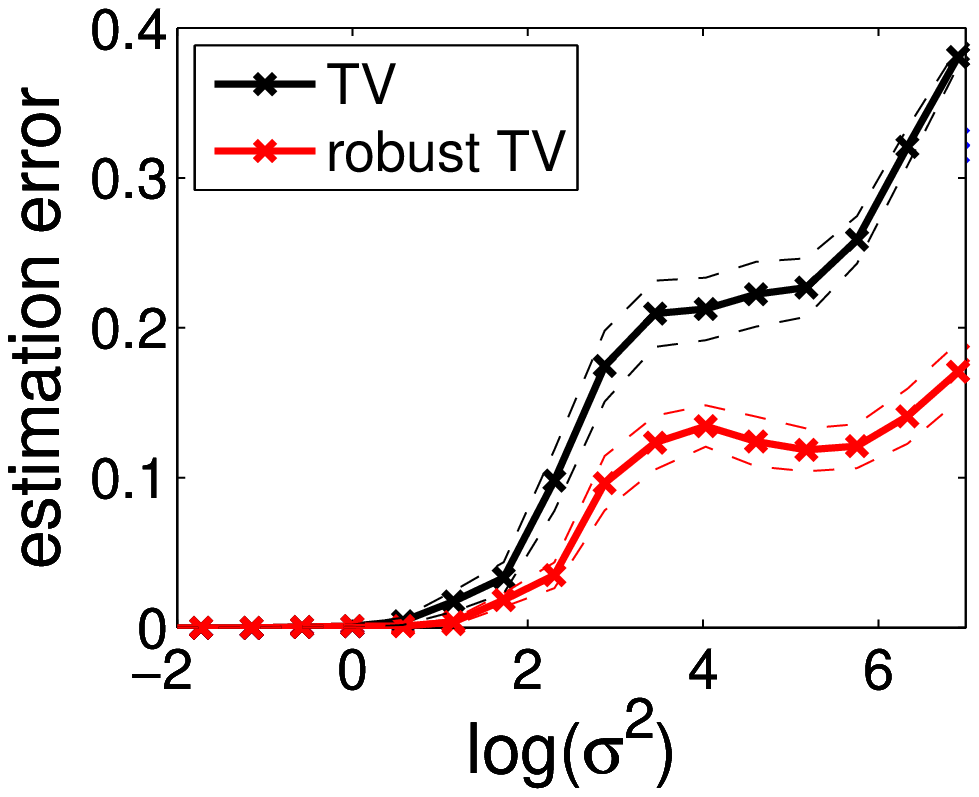} 
\hspace*{-.95cm}

\end{center}

\vspace*{-.6cm}

\caption{\textbf{Three left plots}: Estimation of noisy piecewise constant 1D signal with outliers (indices 5 and 15 in the chain of 20 nodes). Left: original signal. Middle: best estimation with total variation (level sets are not correctly estimated). Right: best estimation with the  \emph{robust} total variation based on noisy cut functions (level sets are correctly estimated, with less bias and with detection of outliers).  \textbf{Right plot}: clustering estimation error vs. noise level, in a sequence of 100 variables, with a single jump, where noise of variance one is added, with $5\%$ of outliers (averaged over 20 replications).}
\label{fig:noisycuts}

 \vspace*{-.25cm}
 
\end{figure}

\textbf{Regular functions and robust total variation.} \hspace*{.05cm}
\label{sec:regular}
By partial minimization, we obtain so-called \emph{regular functions}~\cite{boykov2001fast, chambolle2009total}. One application is ``noisy cut functions'':
  for a given  weight function $d: W \times W \to \rb_+$,  where each node in $W$ is uniquely associated in a node in $V$, we consider the submodular function obtained as the minimum cut adapted to $A$ in the augmented graph (see right plot of \myfig{speed}):
$
F(A) =  \min_{B \subset W} \  \sum_{k \in B, \ j \in W \backslash B} d(k,j) + \lambda | A \Delta B|$.
 This allows for robust versions of cuts, where some gaps may be tolerated. See examples in \myfig{noisycuts}, illustrating the behavior of the type of graph displayed in the bottom-right plot of \myfig{speed}, where the performance of the robust total variation is significantly more stable in presence of outliers.

\textbf{Cardinality-based functions.} \hspace*{.05cm}
\label{sec:card}
For $F(A) = h(|A|)$ where $h$ is such that $h(0)=h(p)=0$ and $h$ concave, we obtain a submodular function, and a \lova extension that depends on the order statistics of $w$, i.e., if $w_{j_1} \geqslant \dots \geqslant w_{j_p}$, then $f(w) = \sum_{k=1}^{p-1} h(k) ( w_{j_k} - w_{j_{k+1}})$. While these examples do not provide significantly different behaviors for the non-decreasing submodular functions explored by~\cite{bach2010structured} (i.e., in terms of \emph{support}), they lead to interesting behaviors here in terms of \emph{level sets}, i.e., they will make the components $w$ cluster together in specific ways.  Indeed, as shown in \mysec{theory}, allowed constant sets $A$ are such that $A$ is inseparable for the function $C \mapsto h(|B \cup C|) - h(|B|)$ (where $B \subset V$ is the set of components with higher values than the ones in $A$), which imposes that  the concave function $h$ is not linear on $[|B|,|B|\!+\!|A|]$. We consider the following examples:

\vspace*{-.1cm}

\BNUM
\setlength{\leftmargin}{0pt}
\item $F(A) = |A| \cdot | V \backslash A| $, leading to
$f(w) = \sum_{i,j=1}^p |w_i - w_j|$. This function can thus be also seen as the cut in the fully connected graph. All patterns of level sets are allowed as the function~$h$ is strongly concave (see left plot of \myfig{card}). This function has been extended in~\cite{toby} by considering situations where each $w_j$ is a vector, instead of a scalar, and replacing the absolute value $| w_i - w_j|$ by any norm $\| w_i - w_j\|$, leading to convex formulations for clustering.

\vspace*{-.05cm}

\item $F(A) =  1$  if $A \neq \varnothing$ and $A \neq V$, and $0$ otherwise, leading to
$f(w) = \max_{i,j} | w_i - w_j|$. Two large level sets at the top and bottom, all the rest of the variables are in-between and separated (\myfig{card}, second plot from the left).

\vspace*{-.05cm}

\item $F(A) =  \max\{ |A|,  | V \backslash A| \} $. This function is piecewise affine, with only one kink, thus only one level set of cardinalty greater than one (in the middle) is possible, which is observed in \myfig{card} (third plot from the left). This may have applications to multivariate outlier detection by considering extensions similar to~\cite{toby}.


%


\ENUM

\vspace*{-.195cm}

\section{Optimization Algorithms} 
\label{sec:algo}
\label{sec:algorithms}
\label{sec:opt}

\vspace*{-.15cm}

In this section, we present optimization methods for minimizing convex objective functions regularized by the \lova extension of a submodular function. These lead to convex optimization problems, which we tackle using proximal methods~(see, e.g., \cite{beck2009fast}). We first start by mentioning that subgradients may easily be derived (but subgradient descent is here rather inefficient as shown in \myfig{speed}).  Moreover, note that with the square loss, the regularization paths are piecewise affine, as a direct consequence of regularizing by a polyhedral function.

\textbf{Subgradient.} \hspace*{.05cm}
From $f(w) = \max_{ s \in B(F)} s^\top w$ and the greedy algorithm\footnote{The greedy algorithm to find extreme points of the base polyhedron should not be confused with the greedy algorithm (e.g., forward selection) that is common in supervised learning/statistics.} presented in \mysec{submodular}, one can easily get in \emph{polynomial time} one subgradient as one of the maximizers $s$. This allows to use subgradient descent, with slow convergence compared to proximal methods (see \myfig{speed}).

\begin{figure}
\begin{center}

 \vspace*{-.25cm}
 
\hspace*{-2.2cm}
\includegraphics[scale=.28]{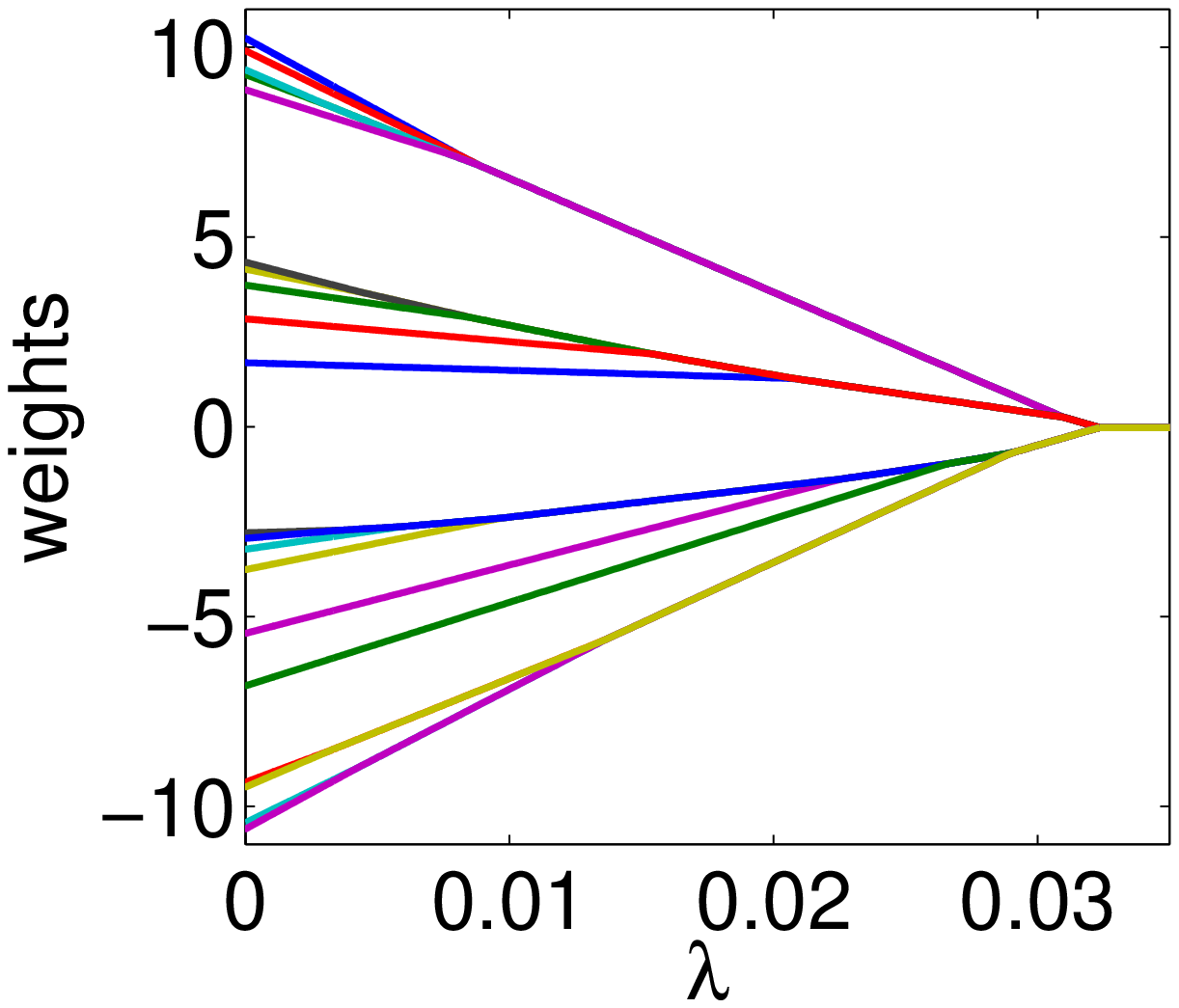}  
\includegraphics[scale=.28]{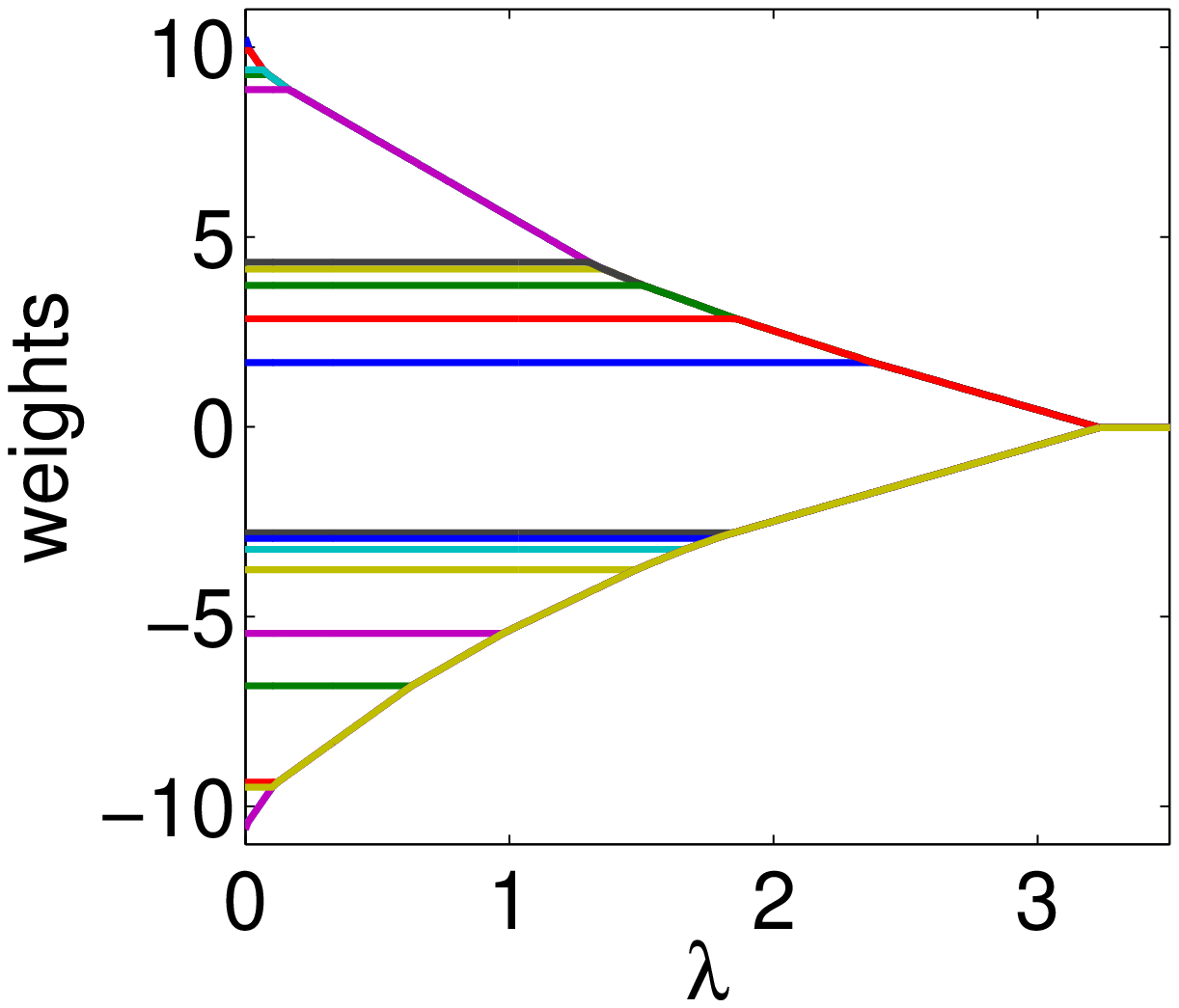}  
\includegraphics[scale=.28]{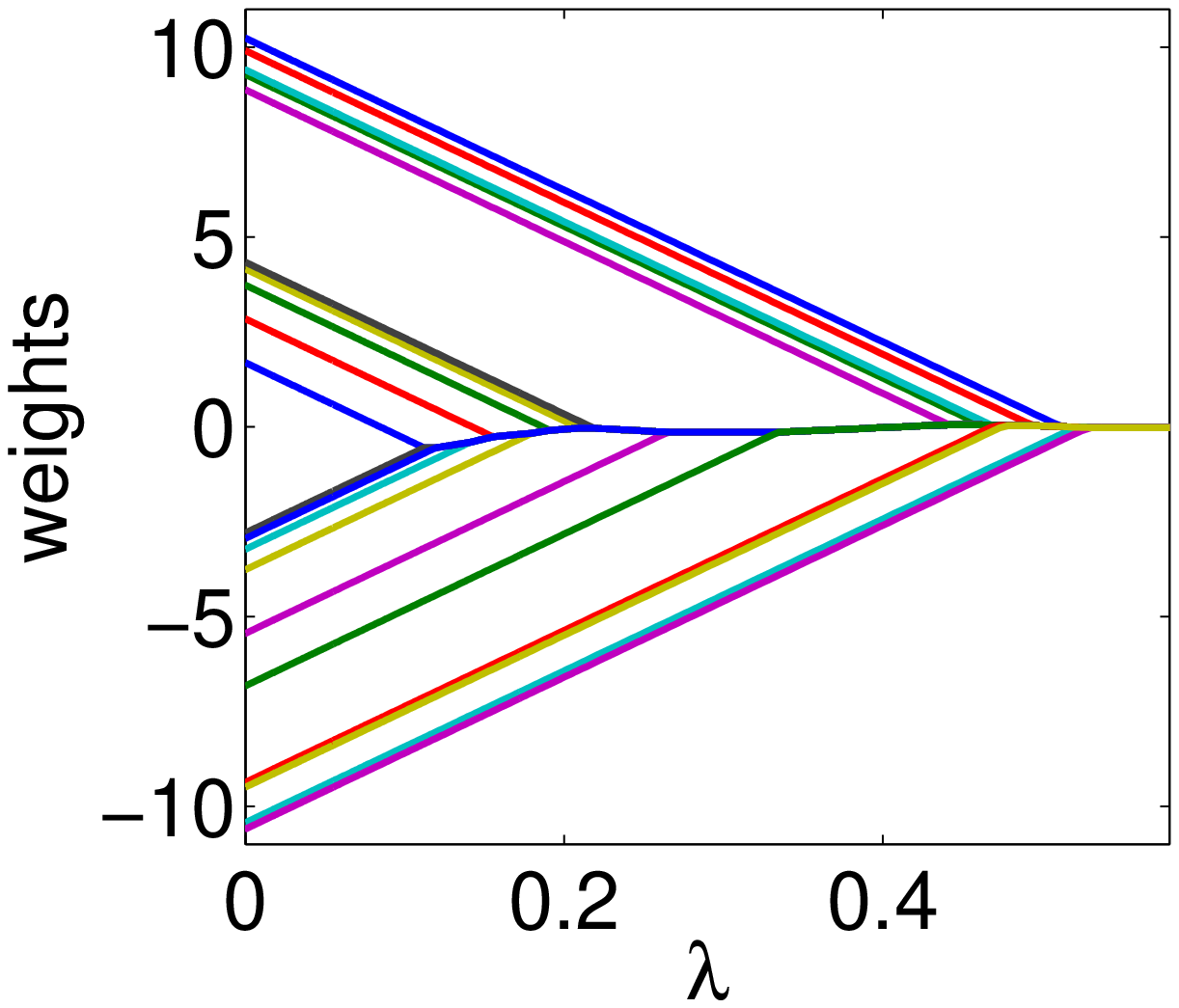} 
\includegraphics[scale=.28]{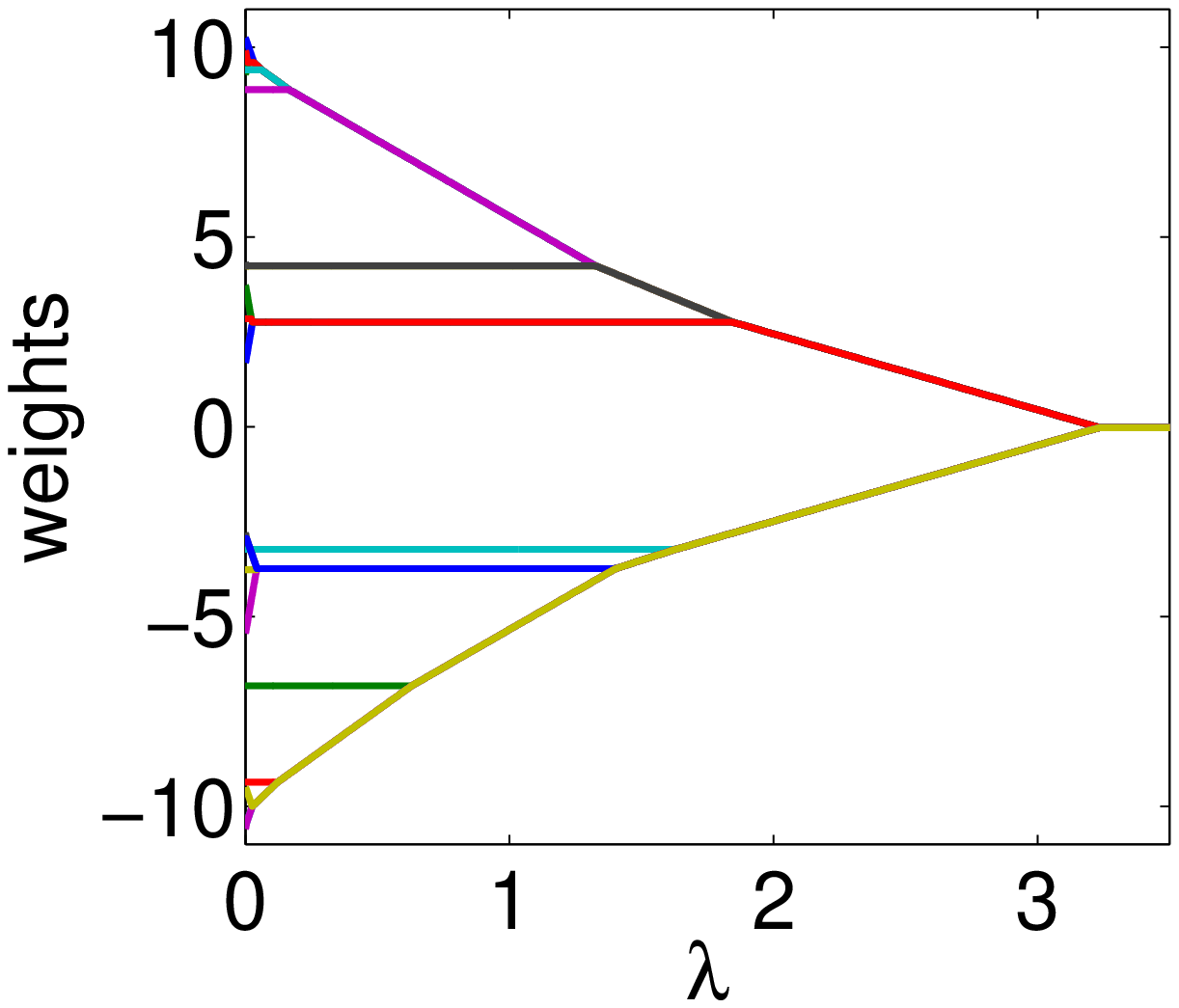} 
\hspace*{-1.5cm}
%

\end{center}

\vspace*{-.5cm}

\caption{\textbf{Left:} Piecewise linear regularization paths of  proximal problems (\eq{prox}) for different functions of cardinality. From left to right: quadratic function (all level sets allowed),  second example in \mysec{card}  (two large level sets at the top and bottom),   piecewise linear with two pieces (a single large level set in the middle). \textbf{Right:} Same plot for the one-dimensional total variation.  Note that in both cases the regularization paths for orthogonal designs are \emph{agglomerative} (see \mysec{opt}), while for general designs, they would still be piecewise affine but not agglomerative. }
\label{fig:card}

\end{figure}

\textbf{Proximal problems through sequences of submodular function minimizations (SFMs).} \hspace*{.05cm}
Given regularized problems of the form $\min_{w \in \rb^p} L(w) + \lambda f(w)$, where $L$ is differentiable with Lipschitz-continuous gradient, \emph{proximal methods} have been shown to be particularly efficient first-order methods~(see, e.g.,~\cite{beck2009fast}). In this paper, we use the method ``ISTA'' and its accelerated variant ``FISTA''~\cite{beck2009fast}.
To apply these methods, it suffices to be able to
solve efficiently: 
\BEQ
\label{eq:prox}
\min_{ w \in \rb^p} \textstyle \frac{1}{2} \| w - z \|_2^2 + \lambda f(w),
\EEQ
 which we refer to as the \emph{proximal problem}. It is known that solving the proximal problem is related to  submodular function minimization (SFM).
More precisely, the minimum of $A \mapsto \lambda F(A) - z(A)$ may be obtained by selecting negative components of the solution of a single proximal problem~\cite{fujishige2005submodular,submodular_tutorial}. Alternatively, the solution of the proximal problem may be obtained by a sequence of at most~$p$ submodular function minimizations of the form $A \mapsto \lambda F(A) - z(A)$, by a decomposition algorithm adapted from~\cite{groenevelt1991two}, and described in~\cite{submodular_tutorial}.

Thus, computing the proximal operator has polynomial complexity since SFM has polynomial complexity. However, it may be  too slow for practical purposes, as the best \emph{generic} algorithm has complexity $O(p^6)$~\cite{orlin2009faster}\footnote{Note that even in the case of symmetric submodular functions, where more efficient algorithms  in $O(p^3)$ for submodular function minimization (SFM) exist~\cite{queyranne1998minimizing}, the minimization of functions of the form $\lambda F(A) - z(A)$ is provably as hard as general SFM~\cite{queyranne1998minimizing}.}. Nevertheless, this strategy is efficient
 for families of submodular functions for which dedicated fast algorithms exist:
\BIT
\item[--] \textbf{Cuts}: Minimizing the cut or the partially minimized cut, plus a modular function, may be done with a min-cut/max-flow algorithm~(see, e.g., \cite{boykov2001fast,chambolle2009total}). 
 For proximal methods,  we need in fact to solve an instance of a \emph{parametric max-flow} problem, which may be done using other efficient dedicated algorithms~\cite{gallo1989fast,chambolle2009total} than the decomposition algorithm derived from~\cite{groenevelt1991two}.
\item[--] \textbf{Functions of cardinality}: minimizing functions of the form $A \mapsto \lambda F(A) - z(A)$ can be done in closed form by sorting the elements of $z$.
\EIT

\textbf{Proximal problems through minimum-norm-point algorithm.} \hspace*{.05cm}
In the \emph{generic} case (i.e., beyond cuts and cardinality-based functions), we can follow~\cite{bach2010structured}: since $f(w)$ is expressed as a minimum of linear functions, the problem reduces to the projection on the polytope $B(F)$, for which we happen to be able to easily  maximize linear functions (using the greedy algorithm described in \mysec{submodular}). This  can be tackled efficiently by the minimum-norm-point algorithm~\cite{fujishige2005submodular}, which iterates between orthogonal projections on affine subspaces and the greedy algorithm for the submodular function\footnote{Interestingly, when used for submodular function minimization (SFM), the minimum-norm-point algorithm has no complexity bound but is empirically faster than algorithms with such bounds~\cite{fujishige2005submodular}.}. We compare all optimization methods on synthetic examples in \myfig{speed}.

\textbf{Proximal path as agglomerative clustering.} \hspace*{.05cm}
When $\lambda$ varies from zero to $+\infty$, then the unique optimal solution of \eq{prox} goes from $z$ to a constant. We now provide conditions under which the regularization path of the proximal problem may be obtained by agglomerative clustering (see examples in \myfig{card}):
\begin{proposition}[Agglomerative clustering]
\label{prop:agglo}
Assume that for all sets $A,B$ such that $B \cap A = \varnothing$ and  $A$ is inseparable for $D \mapsto F(B \cup D) - F(B)$, we have:

\vspace*{-.35cm}

\BEQ
\label{eq:agglo}
\textstyle \forall C \subset A, \ \frac{|C|}{|A|} [ F( B \cup A) - F(B) ]
\leqslant   F( B \cup C) - F(B) .
\EEQ
 
\vspace*{-.2cm}

Then the   regularization path for \eq{prox} is \emph{agglomerative}, that is, if two variables are in the same constant for a certain $\mu \in \rb_+$, so are they for all larger $\lambda \geqslant \mu$.
\end{proposition}
As shown in the supplementary material, the assumptions required for by Prop.~\ref{prop:agglo} are satisfied by (a) all submodular set-functions that only depend on the cardinality, and (b) by the one-dimensional total variation---we thus recover and extend known results from~\cite{harchaoui2008catching,hoefling910path,toby}.

\textbf{Adding an $\ell_1$-norm.} \hspace*{.05cm}
Following~\cite{tibshirani2005sparsity}, we may add the $\ell_1$-norm $\| w\|_1$ for additional sparsity of $w$ (on top of shaping its level sets). The following proposition extends the result for the one-dimensional total variation~\cite{tibshirani2005sparsity,mairal2010online} to all submodular functions and their \lova extensions:
\begin{proposition}[Proximal problem for $\ell_1$-penalized problems]
\label{prop:proxL1}
The unique minimizer of $\frac{1}{2} \| w- z \|_2^2 + f(w) + \lambda \| w\|_1$  may be obtained by soft-thresholding
the minimizers of $\frac{1}{2} \| w- z  \|_2^2 + f(w)$. That is, the proximal operator for $f + \lambda \| \cdot \|_1$ is equal to the composition of the proximal operator for $f$ and the one for $\lambda \| \cdot \|_1$.
\end{proposition}

\begin{figure}

\vspace*{-.45cm}

\begin{center}
\hspace*{1cm}
\includegraphics[scale=.35]{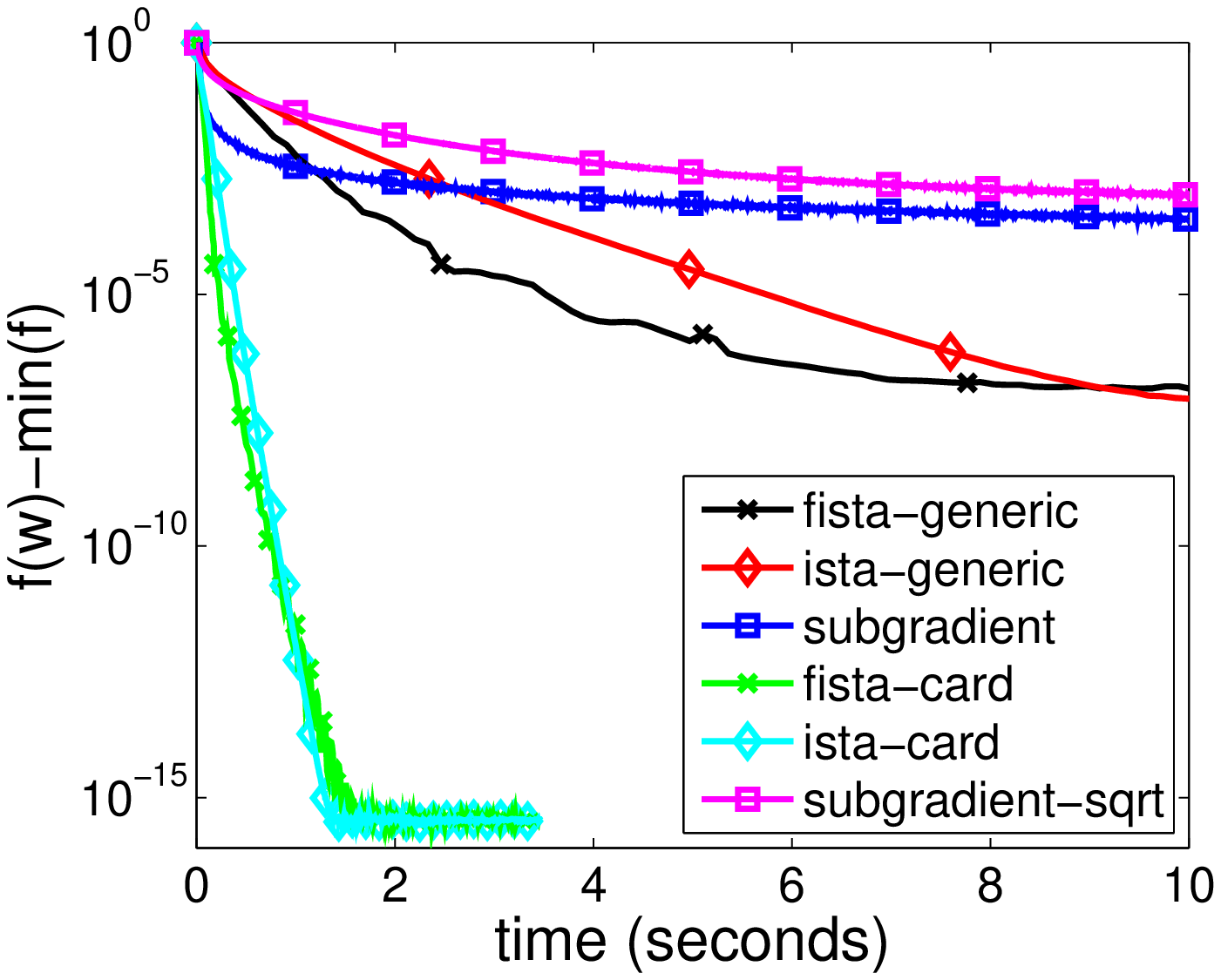} \hspace*{1cm}
\parbox[b]{5cm}{\hspace*{.1cm} \includegraphics[scale=.4]{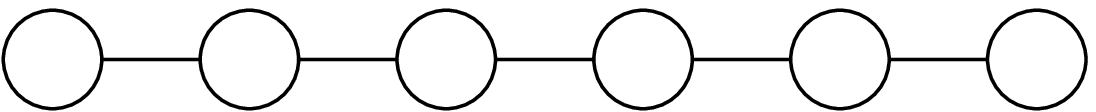}  \\[.1cm]
\includegraphics[scale=.4]{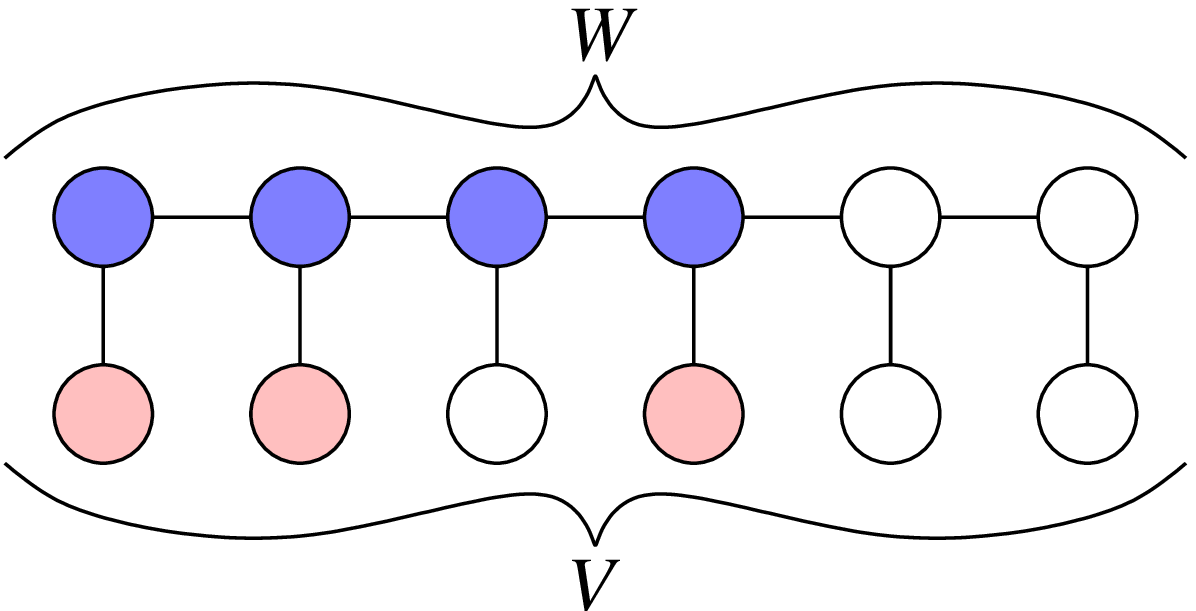} \vspace*{.2cm}}
\end{center}

\vspace*{-.5cm}

\caption{\textbf{Left}: Matlab running times of different optimization methods on 20 replications of a least-squares regression problem with $p=1000$ for a cardinality-based submodular function (best seen in color). Proximal methods with the generic algorithm (using the minimum-norm-point algorithm) are faster than subgradient descent (with two schedules for the learning rate, $1/t$ or $1/\sqrt{t}$). Using the dedicated algorithm (which is not available in all situations) is significantly faster. \textbf{Right}: Examples of graphs (top:   chain graph, bottom: hidden chain graph, with sets $W$ and $V$ and examples of a set $A$ in light red, and $B$ in blue, see text for details).}
\label{fig:speed}
\end{figure}

\section{Sparsity-inducing Properties}
\label{sec:theory}

\vspace*{-.2cm}

Going from the penalization of supports to the penalization of level sets introduces some complexity and for simplicity in this section, we only consider the analysis in the context of orthogonal design matrices, which is often referred to as the denoising problem, and in the context of level set estimation already leads to interesting results. That is, we study the global minimum of the proximal problem in \eq{prox} and make some assumption regarding $z$ (typically $z = w^\ast + \mbox{ noise} $), and provide guarantees related to the recovery of the level sets of $w^\ast$. We first start by characterizing the allowed level sets, showing that the partial constraints defined in \mysec{properties} on faces of $\{f(w) \leqslant 1\}$ do not create by chance further groupings of variables (see proof in supplementary material).
\begin{proposition}[Stable constant sets]
\label{prop:patterns}
Assume  $z \in \rb^p$  has an absolutely continuous density with respect to
the Lebesgue measure. Then, with probability one,  the unique minimizer $\hat{w}$ of \eq{prox} has constant sets that define a partition corresponding to a lattice $\mathcal{D}$ defined in Prop.~\ref{prop:faces}.
  \end{proposition}
We now show that under certain conditions the recovered constant sets are the correct ones:\\[-.5cm]
\begin{theorem}[Level set recovery]
\label{theo:support}
Assume that $z = w^\ast + \sigma \varepsilon$, where $\varepsilon \in \rb^p $ is a standard Gaussian random vector, and $z^\ast$ is consistent with the lattice $\mathcal{D}$ and its associated poset $\Pi(\mathcal{D}) = (A_1,\dots,A_m)$, with values $v^\ast_j$ on $A_j$, for $j \in \{1,\dots,m\}$. Denote $B_j = A_1 \cup \cdots \cup A_j$ for $j \in \{1,\dots,m\}$.
Assume that there exists some constants $\eta_j>0$ and $\nu>0$  such that:
\BEA
\label{eq:eta}
&\textstyle
\hspace*{-.45cm}
 \forall C_j \!\subset\! A_j, 
F(B_{j-1} \!\cup\! C_j) \!-\!  F(B_{j-1} ) \!-\! \frac{ |C_j|}{|A_j|}  [ F(B_{j-1}\! \cup \! A_j) \!-\! F(B_{j-1}) ]
\geqslant \eta_j \min \! \big\{
\frac{ |C_j|}{|A_j|}  , 1 \!-\!\frac{ |C_j|}{|A_j|}  
\big\}, \ \
\\[-.1cm]
&\label{eq:rho}
\forall i,j \in \{1,\dots,m\} , \  A_i \succcurlyeq A_j  \Rightarrow  v_i^\ast - v_j^\ast \geqslant \nu,
\\[-.1cm]
&
\label{eq:lambda}
\textstyle
\forall j \in \{1,\dots,m\} , \   \lambda  \big| \frac{F(B_{j} ) - F(B_{j-1})}{|A_j|} \big|
\leqslant \nu / 4.\\[-.4cm]
\nonumber
\EEA
Then 
the unique minimizer $\hat{w}$ of \eq{prox}  is associated to the same lattice  $\mathcal{D}$ than $w^\ast$,
with probability greater than
$
\textstyle 1 - \sum_{j=1}^m \exp\big( -\frac{ \nu^2 |A_j|}{32 \sigma^2} \big) - 2 \sum_{j=1}^m |A_j| \exp\big( - \frac{ \lambda^2 \eta_j^2}{ 2\sigma^2 |A_j|^2 } \big)
$.
\end{theorem}
We now discuss  the three main assumptions of Theorem~\ref{theo:support} as well as the probability estimate:\\[-.6cm]
\BIT
\item[--] \eq{eta} is the equivalent of the support recovery of the Lasso~\cite{Zhaoyu} or its extensions~\cite{bach2010structured}. The main difference is that for support recovery, this assumption is always met for orthogonal designs, while here it is not always met. Interestingly, the validity of level set recovery  implies the agglomerativity of proximal paths (\eq{agglo} in Prop.~\ref{prop:agglo}).

Note that if  \eq{eta} is satisfied only with $\eta_j \geqslant 0$ (it is then exactly \eq{agglo} in Prop.~\ref{prop:agglo}), then, even with infinitesimal noise, one can show that in some cases, the wrong level sets may be obtained with non vanishing probability, while if $\eta_j$ is strictly negative, one can show that in some cases, we \emph{never} get the correct level sets. \eq{eta} is thus essentially sufficient and necessary.

\item[--] \eq{rho} corresponds to having  distinct values of $w^\ast$  far enough from each other.

\item[--]  \eq{lambda} is a constraint on $\lambda$ which controls the bias  of the estimator: if it is too large, then there may be a merging of two clusters.

\item[--] In the probability estimate, the second term is small if all $\sigma^2 |A_j|^{-1}$ are small enough (i.e., given the noise, there is enough data to correctly estimate the values of the constant sets) and the third term is small if  $\lambda$ is large enough, to avoid that clusters split.
\EIT

\textbf{One-dimensional total variation.} \hspace*{.05cm}
In this situation, we always get $\eta_j=0$, but in some cases, it cannot be improved (i.e., the best possible $\eta_j$ is equal to zero), and as shown in the supplementary material, this occurs as soon as there is a ``staircase'', i.e., a piecewise constant vector, with a sequence of at least two consecutive increases, or two consecutive decreases, showing that in the presence of such staircases, one cannot have consistent support recovery, which is a well-known issue in signal processing (typically, more steps are created).
If there is no staircase effect, we have $\eta_j = 1$ and \eq{lambda} becomes
$\lambda \leqslant \frac{\nu}{8}\min_{j} |A_j| $. If we take $\lambda$ equal to the limiting value in \eq{lambda}, then we obtain a probability less than
$ 1 - 4 p \exp ( - \frac{ \nu^2 \min_{j} |A_j|^2}{ 128 \sigma^2 \max_{j} |A_j|^2 }  ) $.
 Note that we could also derive general results when an additional $\ell_1$-penalty is used, thus extending results from~\cite{rinaldo2009properties}. 
 
 \textbf{Two-dimensional total variation.} \hspace*{.05cm}
 In this situation, even with only two different values for $z^\ast$, then we may have $\eta_j < 0$, leading to additional problems, which has already been noticed in continuous settings (see, e.g.,~\cite{duval:154} and the supplementary material).

 \textbf{Clustering with $F(A) = |A| \cdot | V \backslash A |$.} \hspace*{.05cm}  In this case, we have
 $\eta_j = |A_j|/2$, and \eq{lambda} becomes $\lambda \leqslant \frac{\nu}{4p}$, leading to the probability of correct support estimation greater than 
$ 1 - 4 p \exp\big( - \frac{  \nu^2}{ 128 p \sigma^2   } \big)$. This indicates that the noise variance $\sigma^2$ should be small compared to $1/p$, which is not satisfactory and would be corrected with the weighting schemes proposed in~\cite{toby}.
 
\section{Conclusion}

\vspace*{-.1cm}

We have presented a family of sparsity-inducing norms dedicated to incorporating prior knowledge or structural constraints on the level sets of linear predictors. We have provided a set of common algorithms and theoretical results, as well as simulations on synthetic examples illustrating the behavior of these norms. Several avenues are worth investigating: first, we could follow current practice in sparse methods, e.g., by considering related adapted concave penalties to enhance sparsity-inducing capabilities, or by extending some of the concepts for norms of matrices, with potential applications in matrix factorization~\cite{Srebro2005Maximum} or multi-task learning~\cite{argyriou2008convex}.

\subsection*{Acknowledgements}
 This paper was partially supported by   the Agence Nationale de la Recherche (MGA Project), the European Research Council (SIERRA Project) and Digiteo (BIOVIZ project).

 \newpage

\appendix

\section{Proof of Proposition~\ref{prop:envelope}}

\begin{proof}
For any $w \in \rb^p$, level sets of $w$ are characterized by an ordered partition $(A_1,\dots,A_m)$ so that $w$ is constant on each $A_j$, with value $t_j$, $j=1,\dots,m$, and so that $(t_j)$ is a strictly decreasing sequence. We can now decompose minimization with respect to $w$ using these ordered partitions and $(t_j)$.

In order to compute the convex envelope, we simply need to compute twice the Fenchel conjugate of the function we want to find the envelope of~(see, e.g., \cite{boyd,borwein2006caa} for definitions and properties of Fenchel conjugates).

Let $s \in \rb^p$; we consider the function $g:w \mapsto \max_{ \alpha \in \rb}  F( \{ w \geqslant \alpha \} )$, and we compute its Fenchel conjugate:
\BEAS
& & g^\ast(s)  \\
& \eqdef & \max_{ w \in [0,1]^p + \rb 1_V  } w^\top s - g(w), \\
& = & \max_{(A_1,\dots,A_m)  \  { \rm    partition} }  \  \bigg\{  \ \max_{ t_1 > \cdots > t_m, \ t_1 - t_m \leqslant 1 }
\sum_{j=1}^{m}t_j s(A_j)   - \max_{j \in \{1,\dots,m\} } F(A_1 \cup \cdots \cup A_j)  \bigg\},\\
& = & \max_{(A_1,\dots,A_m) \  { \rm    partition} }  \  \bigg\{  \ \max_{ t_1 > \cdots > t_m, \ t_1 - t_m \leqslant 1 }
\sum_{j=1}^{m-1} (t_j - t_{j+1} ) s(A_1 \cup \cdots \cup A_j) + t_m s(V) \\
& & \hspace*{5cm} - \max_{j \in \{1,\dots,m\} } F(A_1 \cup \cdots \cup A_j)  \bigg\}  
\mbox{ by integration by parts,}\\
& = & \iota_{ s(V) =0 } (s) +  \max_{(A_1,\dots,A_m) \  { \rm     partition} }
\bigg\{
\max_{j \in \{1,\dots,m-1\} } s(A_1 \cup \cdots \cup A_j) - \max_{j \in \{1,\dots,m\} } F(A_1 \cup \cdots \cup A_j)
\bigg\},
\\& = & \iota_{ s(V) =0 } (s) +  \max_{(A_1,\dots,A_m) \  { \rm     partition} }
\bigg\{
\max_{j \in \{1,\dots,m-1\} } s(A_1 \cup \cdots \cup A_j) - \max_{j \in \{1,\dots,m-1\} } F(A_1 \cup \cdots \cup A_j)
\bigg\},
\EEAS
where $\iota_{ s(V) =0 } $ is the indicator function of the set $\{ s(V) = 0\}$ (with values $0$ or $+\infty$). Note that $\max_{j \in \{1,\dots,m\} } F(A_1 \cup \cdots \cup A_j)
  = \max_{j \in \{1,\dots,m-1\} } F(A_1 \cup \cdots \cup A_j) $ because $F(V) = 0$.

Let $h(s) =\iota_{ s(V) =0 }(s)  +  \max_{A \subset V}  \{ s(A) - F(A) \} $. We clearly have $g^\ast(s) \geqslant h(s)$, because we take a maximum over a larger set (consider $m=2$). Moreover, for all partitions $(A_1,\dots,A_m)$, if $s(V)=0$,
$\max_{j \in \{1,\dots,m-1\} } s(A_1 \cup \cdots \cup A_j)  \leqslant 
\max_{j \in \{1,\dots,m-1\} } ( h(s) + F(A_1 \cup \cdots \cup A_j) )
= h(s) + \max_{j \in \{1,\dots,m-1\} } F(A_1 \cup \cdots \cup A_j)$, which implies that $g^\ast(s) \leqslant h(s)$. Thus $g^\ast(s) = h(s)$.

Moreover, we have, since $f$ is invariant by adding constants and $f$ is submodular,
\BEAS
 \max_{ w \in [0,1]^p + \rb 1_V  } w^\top s - f(w) & = & 
 \iota_{ s(V) =0 } (s) +  \max_{ w \in [0,1]^p   }  \{w^\top s - f(w) \} \\ 
 & = & 
 \iota_{ s(V) =0 } (s) +  \max_{ A \subset V  } \{s(A) - F(A) \}  = h(s),
\EEAS
where we have used the fact that minimizing a submodular function is equivalent to minimizing its \lova extension on the unit hypercube.
Thus $f$ and $g$ have the same Fenchel conjugates. The result follows from the convexity of $f$, using the fact the convex envelope is the Fenchel bi-conjugate~\cite{boyd,borwein2006caa}.
\end{proof}

\section{Proof of Proposition~\ref{prop:extreme}}

\begin{proof} Extreme points of $\mathcal{U}$ correspond to full-dimensional faces of $B(F)$. From Corollary 3.4.4 in \cite{fujishige2005submodular}, these facets are exactly the ones that correspond to sets $A$ with the given conditions. These facets are defined as the intersection of $\{s(A)=F(A)\}$ and $\{ s(V) = F(V) \}$,  which leads to the desired result. Note that this is also a consequence of Prop.~\ref{prop:faces}. Note that when $F$ is symmetric, the second condition is equivalent to $V \backslash A$ being inseparable for $F$.
\end{proof}

\section{Proof of Proposition~\ref{prop:faces}}
 \begin{proof}
Given that the polyhedra $\mathcal{U}$ and $B(F)$ are polar to each other~\cite{rockafellar97}, the proposition follows from Theorem 3.43 in \cite{fujishige2005submodular}, where each of our three assumptions are equivalent to a corresponding one in Theorem 3.43 from \cite{fujishige2005submodular}.
\end{proof}

\section{Proof of Proposition~\ref{prop:agglo}}
\label{app:agglo}

We first start by a lemma, which follows common practice in sparse recovery (assume a certain sparsity pattern and check when it is actually optimal):
\begin{lemma}[Optimality of lattice for proximal problem]
\label{lemma:opt}

The solution of the proximal problem in \eq{prox}
corresponds to a lattice $\mathcal{D}$ if and only if
$v = (M^\top M)^{-1} ( M^\top z - \lambda t)$ satisfies the order relationships imposed by $\mathcal{D}$ and
$$\frac{1}{\lambda}( \idm - M (M^\top M)^{-1} M^\top ) z + M (M^\top M)^{-1} t \in B(F),$$
where $M\in \rb^{p \times m}$ is the indicator matrix of the partition $\Pi(\mathcal{D})$, and $t_i = F( A_1 \cup \cdots \cup A_i) -  F( A_1 \cup \cdots \cup A_{i-1})$, $i=1,\dots,m$.
\end{lemma}
\begin{proof}
Any $w \in \rb^p$ belongs to a single face relative interior from Prop.~\ref{prop:faces}, defined by a lattice $\mathcal{D}$, i.e., $w$ is  constant on $A_i$ with value $v_i$ (which implies that $w = Mv$) and such that $v_i > v_{j}$ as soon as $A_i \succcurlyeq  A_j$. We assume a topological ordering of the sets $A_i$, i.e, $A_i \succcurlyeq A_j \Rightarrow i \geqslant j$. Since the \lova extension is linear for $w$ in $\mathcal{U}_\mathcal{D}$ (and equal to $t^\top v$ for $w = Mv$), the optimum over $w$ can be found by minimizing  with respect to $v$
$$
\frac{1}{2} \| z - M v\|_2^2 + \lambda t^\top v.
$$
 We thus get, by setting the gradient to zero:
$$v = (M^\top M)^{-1} ( M^\top z - \lambda t).$$ 

Optimality conditions for $w$ for \eq{prox} are that $w - z + \lambda s = 0$, for $s \in B(F)$ and $f(w) = w^\top s$
(these are obtained from general optimality conditions for functions defined as pointwise maxima~\cite{borwein2006caa}). 
Thus our candidate $w = Mv$ is optimal if and only if $Mv - z + \lambda s = w - z + \lambda s = 0$ for (a)  $s \in B(F)$ and (b) $f(w) = w^\top s$. From Prop.~10 in~\cite{submodular_tutorial}, for (b)  to be valid, $s \in B(F)$ simply has to satisfy $s(A_1 \cup \cdots \cup A_i) = F(A_1 \cup \cdots \cup A_i)$ for all $i$.

Note that
$$
z - M v = ( \idm - M (M^\top M)^{-1} M^\top ) z + \lambda M (M^\top M)^{-1} t,
$$
and that for all $i \in \{1,\dots,m\}$,
$$1_{A_i}^\top  ( \idm - M  (M^\top M)^{-1}  M^\top ) z =
\delta_i^\top M^\top  ( \idm - M  (M^\top M)^{-1}  M^\top ) z =
0, $$ 
where $\delta_i$ is indicator vector of the singleton $\{i\}$. Moreover, we have 
$$1_{A_i}^\top M (M^\top M)^{-1} t = t_i = F( A_1 \cup \cdots \cup A_i) -  F( A_1 \cup \cdots \cup A_{i-1}),$$ so that, if $B_i =A_1 \cup \cdots \cup A_i$, 
$ [ ( \idm - M  (M^\top M)^{-1}  M^\top ) z ] (B_i) = 0$, $[ M (M^\top M)^{-1} t] (B_i) = F(B_i)$, for all $i \in \{1,\dots,m\}$. This implies that $ \big[\frac{1}{\lambda}(z - Mv) \big](A_i) = t_i$, and thus 
$ \big[\frac{1}{\lambda}(z - Mv) \big](B_i) = F(B_i)$.

 Thus, if (a) is satisfied, then (b) is always satisfied. Thus to check if a certain lattice leads to the optimal solution, we simply have to check that 
$\frac{1}{\lambda}( \idm - M (M^\top M)^{-1} M^\top ) z + M (M^\top M)^{-1} t \in B(F)$.
\end{proof}

 We now turn to the proof of Proposition~\ref{prop:agglo}.
 
\begin{proof}
We show that when $\lambda$ increases, we move to a lattice which has to be merging some constant sets. Let us assume that a lattice $\mathcal{D}$ is optimal for a certain $\mu$. Then,
from Lemma~\ref{lemma:opt}, we have
$$\frac{1}{\mu}( \idm - M (M^\top M)^{-1} M^\top ) z + M (M^\top M)^{-1} t \in B(F).$$

Moreover,  since from Prop.~\ref{prop:faces},  $A_i$ is separable for $C_i \mapsto F(B_{i-1}\cup C_i) - F(B_{i-1})$, from the assumption of the proposition, we obtain:
$$
\forall C_i \subset A_i, \
[ M (M^\top M)^{-1} t] (C_i) 
= \frac{|C_i|}{|A_i| }( F(B_{i-1}\cup A_i) - F(B_{i-1}) )\leqslant F(B_{i-1}\cup C_i) - F(B_{i-1}),
$$
which implies, for all $C \subset V$:
\BEAS
[ M (M^\top M)^{-1} t] (C) & = & \sum_{j=1}^m [ M (M^\top M)^{-1} t] (C \cap A_i) \mbox{ by modularity,}
\\ & \leqslant &  \sum_{i=1}^m \bigg\{ F(B_{i-1}\cup ( C \cap A_i) ) - F(B_{i-1}) \bigg\} \mbox{ from above,} 
\\ & \leqslant &  \sum_{i=1}^m \bigg\{ F((B_{i-1} \cap C)\cup ( C \cap A_i) ) - F(B_{i-1} \cap C) \bigg\} \mbox{ by submodularity,}
\\ & = &  \sum_{i=1}^m \bigg\{  F(B_i \cap C  ) - F(B_{i-1} \cap C) \bigg\}= F(C).
\EEAS

Thus, for any set $C$, we have for $\lambda \geqslant \mu$ (which implies $\frac{\mu}{\lambda} \in [0,1]$),
\BEAS
& & \big[ \frac{1}{\lambda}( \idm - M (M^\top M)^{-1} M^\top ) z + M (M^\top M)^{-1} t \big](C) \\
& = &  \frac{\mu}{\lambda} 
\big[ \frac{1}{\mu}( \idm - M (M^\top M)^{-1} M^\top ) z + M (M^\top M)^{-1} t \big](C) 
+ ( 1- \frac{\mu}{\lambda} ) \big[M (M^\top M)^{-1} t \big](C)  \\
& \leqslant & \frac{\mu}{\lambda}  F(C) +  ( 1- \frac{\mu}{\lambda} ) F(C) = F(C).
\EEAS
Thus the second condition in Lemma~\ref{lemma:opt} is satisfied, thus it has to be the first one which is violated, leading to merging two constant sets.
\end{proof}

We now show that for special cases, the condition in \eq{agglo} is satisfied, and we   also show when the condition in \eq{eta} of Theorem~\ref{theo:support} is satisfied or not:
\BIT
\item \textbf{Cardinality-based functions}: the condition in \eq{agglo} is equivalent to
$$
\frac{ h(|B|+|A|) - h(|B|)}{|A|}  
\leqslant   \frac{ h(|B|+|C|) - h(|B|)}{|C|} ,
$$
which is a consequence of the concavity of $h$. Moreover the condition in \eq{eta} is equivalent to
$$ h(|B|+|C|) - h(|B|)  - \frac{ |C|}{|A|}  [ h(|B|+|A|) - h(|B|)   ]
\geqslant \eta \min \Big\{
\frac{ |C|}{|A|}  , 1 -\frac{ |C|}{|A|}  \Big\}.$$
For $h(t) = t(p-t)$, this is equivalent to
$$
|A| ( |C|-|A|) \geqslant \eta \min \Big\{
\frac{ |C|}{|A|}  , 1 -\frac{ |C|}{|A|}  \Big\},
$$
which is true as soon as $\eta \leqslant |A|/2$.

\item \textbf{One-dimensional total variation}: we assume that we have a chain graph. Note that $A$ must be an interval and that $B$ only enters the problem if one of its elements is a neighbor of one of the two extreme elements of $A$. We thus have eight cases, depending on the three possibilities for these two neighbors of $A$ (in $B$, in $V \backslash B$, or no neighbor, i.e., end of the chain). We consider all 8 cases, where $C$ is a non trivial subset of $A$, and compute a lower bound on $F(B \cup C) - F(B) - \frac{|C|}{|A|}[ F( B \cup A) - F(B) ]$.
\BIT
\item[--] left: $B$, right: $B$. $F(B)=2$, $F(B \cup A)=0$, $F(C \cup B) \geqslant 2$. Bound= $2 \frac{|C|}{|A|}$
\item[--]left: $B$, right: $V \backslash B$. $F(B)=1$, $F(B \cup A)=1$, $F(C \cup B) \geqslant 1$. Bound= $0$

\item[--]left: $B$, right: none. $F(B)=1$, $F(B \cup A)=0$, $F(C \cup B) \geqslant 1$. Bound= $\frac{|C|}{|A|}$

\item[--]left: $V \backslash B$, right: $B$. $F(B)=1$, $F(B \cup A)=1$, $F(C \cup B) \geqslant 1$. Bound= $0$

\item[--]left: $V \backslash B$, right: $V \backslash B$. $F(B)=0$, $F(B \cup A)=2$, $F(C \cup B) \geqslant 2$. Bound= $2 - 2\frac{|C|}{|A|}$

\item[--]left: $V \backslash B$, right: none. $F(B)=0$, $F(B \cup A)=1$, $F(C \cup B) \geqslant 1$. Bound= $1 - \frac{|C|}{|A|}$

\item[--]left: none, right: $B$. $F(B)=2$, $F(B \cup A)=0$, $F(C \cup B) \geqslant 2$. Bound= $\frac{|C|}{|A|}$

\item[--]left: none, right: $V \backslash B$.  $F(B)=1$, $F(B \cup A)=0$, $F(C \cup B) \geqslant 1$. Bound= $\frac{|C|}{|A|}$

\item[--]left: none, right: none. $F(B)=0$, $F(B \cup A)=0$, $F(C \cup B) \geqslant 1$. Bound= $1$.

\EIT
Considering all cases, we get a lower bound of zero, which shows that the paths are agglomerative. However, there are two cases where no strictly positive lower bounds are possible, namely when the two extremities of $A$ have respective neighbors in $B$ and $V \backslash B$. Given that $B$ is a set of higher values for the parameters and $V \backslash (A \cup B)$ is a set of lower values, this is exactly a staircase. When there is no such staircase, we get a lower bound of $\min \{ |A|/|C|, 1 - |A|/|C|\}$, hence $\eta=1$.
 \EIT

\section{Proof of Proposition~\ref{prop:proxL1}}
 
\begin{proof}
We denote by $w$ the unique mininizer of $\frac{1}{2}\|w - z \|_2^2 + f(w)$ and $s$ the associated dual vector in $B(F)$. The optimality conditions are  $w-z + s = 0$,  and $f(w) = w^\top s$ (again from optimality conditions for pointwise maxima).

We assume that $w$ takes distinct values $v_1,\dots,v_m$ on the sets $A_1,\dots,A_m$. We define $t$ as 
$t_k = {\rm sign}(w_k) ( |w_k| - \lambda )_+$ (which is the unique minimizer of $\frac{1}{2} \| w - t\|_2^2 + \lambda \| t \|_1$). The constant sets of $t$ are $A_j$, for $j$ such that $|v_j| > \lambda$ and zero for the union of all $A_j$'s such that $|v_j| \leqslant \lambda$.  Since $t$ is obtained by soft-thresholding $w$, which corresponds to $\ell_1$-proximal problem, we have that $t - w + \lambda q=0$ with $\| q\|_\infty \leqslant 1$ and $q^\top t = \| t\|_1$.

By combining these two equalities, with have
$ t -z + s + \lambda q = 0$ with $\| q\|_\infty \leqslant 1$,  $q^\top t = \| t\|_1$ and $s \in B(F)$. The only remaining element to show that $t$ is optimal for the full problem is that $f(t) = s^\top t$. This is true since the level sets of $w$ are finer than the ones of $t$ (i.e., it is obtained by grouping some values of $w$), with no change of ordering~\cite{submodular_tutorial}.
\end{proof}

\section{Proof of Proposition~\ref{prop:patterns}}

 \begin{proof}
From Lemma~\ref{lemma:opt}, the solution has to correspond to a lattice $\mathcal{D}$ and we only have to show that with probability one, the vector $v = (M^\top M)^{-1} ( M^\top z - \lambda t)$ has distinct components, which is straightforward because it has an absolutely continuous density with respect to the Lebesgue measure.
\end{proof}

\section{Proof of Theorem~\ref{theo:support}}

\begin{proof}
From Lemma~\ref{lemma:opt}, in order to correspond to the same lattice $\mathcal{D}$, we simply need that (a) 
$v = (M^\top M)^{-1} ( M^\top z - \lambda t)$ satisfies the order relationships imposed by $\mathcal{D}$ and that (b)
$$\frac{1}{\lambda}( \idm - M (M^\top M)^{-1} M^\top ) z + M (M^\top M)^{-1} t \in B(F).$$
Condition (a) is satisfied as soon as $\| w - w^\ast\|_\infty \leqslant \nu$, which is implied by
\BEQ
\label{eq:AA}
 \sigma  \| (M^\top   M)^{-1}  M^\top \varepsilon  \|_\infty
 \leqslant \nu / 4 \ \ \mbox{ and }\ \   \| \lambda (M^\top  M)^{-1}   t\|_\infty \leqslant \nu / 4.
 \EEQ
 The second condition in \eq{AA} is met by assumption, while the first one leads to the sufficient conditions
 $   \forall j, \ | \varepsilon(A_j) | \leqslant \nu |A_j| /4 \sigma$, leading  by the union bound to the probabilities
 $\sum_{j=1}^m \exp\big( -\frac{ \nu^2 |A_j|}{32 \sigma^2} \big)$.
 
Following the same reasoning than in the proof of Prop.~\ref{prop:agglo}, condition (b) is satisfied as soon as for all $j \in \{1,\dots,m\}$, and all $C_j \subset A_j$,
$$
\big[ \sigma \frac{1}{\lambda}( \idm - M (M^\top M)^{-1} M^\top ) \varepsilon \big](C_j)
\leqslant \eta_j \min \bigg\{
\frac{ |C_j|}{|A_j|}  , 1 - \frac{ |C_j|}{|A_j|}  
\bigg\}.
$$
Indeed, this implies that for all $j$,
\BEAS
& &\big[ \frac{1}{\lambda}( \idm - M (M^\top M)^{-1} M^\top ) z + M (M^\top M)^{-1} t \big](C_j)\\
& = &  \big[ \frac{\sigma}{\lambda}( \idm - M (M^\top M)^{-1} M^\top ) \varepsilon + M (M^\top M)^{-1} t \big](C_j)\\
& \leqslant  &  \eta_j \min \bigg\{
\frac{ |C_j|}{|A_j|}  , 1 - \frac{ |C_j|}{|A_j|}  
\bigg\} + \frac{|C_j|}{|A_j| }( F(B_{j-1}\cup A_i) - F(B_{j-1}) )\\
& \leqslant  &    F(B_{j-1}\cup C_j) - F(B_{j-1}) ,
\EEAS
which leads to $\big[ \frac{1}{\lambda}( \idm - M (M^\top M)^{-1} M^\top ) z + M (M^\top M)^{-1} t \big] \in B(F)$ using the sequence of inequalities used in the proof of Prop.~\ref{prop:agglo}.

From Lemma~\ref{lemma:conc} below, we thus get the probability
$
2 \sum_{j=1}^m |A_j| \exp \Big( - \frac{ \lambda^2 \eta_j^2}{ 2\sigma^2 |A_j|^2 } \Big).
$
 \end{proof}
\begin{lemma}
\label{lemma:conc}
For $F(A) = \min \big\{ \frac{|A|}{p}, 1 - \frac{|A|}{p} \big\}$, and $s$ normal with mean zero and variance $I - \frac{1_V1_V^\top}{p}$, we have:
$$\mathbb{P} \Big( \max_{A \subset V, A \neq \varnothing, A \neq V} \frac{s(A)}{F(A)} \geqslant t
\Big) \leqslant 2 p \exp \Big( - \frac{ t^2}{2 p^2} \Big).$$
\end{lemma}
\begin{proof}
Since $F$ depends on uniquely on the cardinality $|A|$ and is symmetric we have, with $\tilde{s} \in \rb^p$ the sorted (in descending order) components of $s$, and $h(a) = \min\{ a/p, 1- a/p\}$:
\BEAS
 & & \mathbb{P}\big( \max_{A \subset V, A \neq \varnothing, A \neq V} \frac{s(A)}{F(A)} \geqslant t
\big) \\
& = & \mathbb{P}\big( \max_{k \in \{1,\dots,p-1\} } \frac{\tilde{s}(\{1,\dots,k\})}{h(k)} \geqslant t
\big) 
\\
& \leqslant  & \mathbb{P}\big( \max_{k \in \{1,\dots,\lfloor p/2\rfloor-1\} } \frac{\tilde{s}(\{1,\dots,k\})}{h(k)} \geqslant t
\big) + \mathbb{P}\big( \max_{k \in \{ \lfloor p/2\rfloor,\dots,p-1\} } \frac{\tilde{s}(\{1,\dots,k\})}{h(k)} \geqslant t
\big) \\
& \leqslant  & 2 \mathbb{P}\big( \max_{k \in \{1,\dots,\lfloor p/2\rfloor-1\} } \frac{\tilde{s}(\{1,\dots,k\})}{k/p} \geqslant t
\big) \mbox{ because of symmetry due to the covariance of $s$ }
 \\
 & \leqslant  & 2 \mathbb{P}\big( \max_{k \in \{1,\dots,\lfloor p/2\rfloor-1\} } \frac{\tilde{s}(\{1,\dots,k\})}{k/p} \geqslant t
\big) 
 \\
 & \leqslant  & 2 \mathbb{P}\big( \max_{k \in \{1,\dots,p \} } s_k \geqslant t/p
\big) 
 \\
 & \leqslant  & 2 p \exp( -   t^2/ 2 p^2) .
\EEAS
\end{proof}

We now consider the three special cases:
\BIT
\item \textbf{One-dimensional total variation}: without the staircase effect, as shown in Appendix~\ref{app:agglo}, we have $\eta_j=1$. Moreover, $|F(B_{j})-F(B_{j-1})| \leqslant 2$, and thus \eq{lambda} leads to $\lambda \leqslant \frac{\nu}{8} \min_{j} |A_j|$. Using the largest possible $\lambda$ in \eq{lambda}, we obtain a probability greater than
\BEAS
&& 1 - \sum_{j=1}^m \exp\big( -\frac{ \nu^2 |A_j|}{32 \sigma^2} \big) - 2 \sum_{j=1}^m |A_j| \exp\big( - \frac{ \lambda^2 \eta_j^2}{ 2\sigma^2 |A_j|^2 } \big) \\
& \geqslant & 1 - \sum_{j=1}^m \exp\big( -\frac{ \nu^2 |A_j|}{32 \sigma^2} \big) - 2 \sum_{j=1}^m |A_j| \exp\big( - \frac{ \nu^2 \min_{j} |A_j|^2}{ 128 \sigma^2 \max_{j} |A_j|^2 } \big) \\
& \geqslant & 1 - \sum_{j=1}^m \exp\big( -\frac{ \nu^2 |A_j|}{32 \sigma^2} \big) - 2p \exp\big( - \frac{ \nu^2 \min_{j} |A_j|^2}{ 128 \sigma^2 \max_{j} |A_j|^2 } \big) \\
& \geqslant & 1 - 4 p \exp\big( - \frac{ \nu^2 \min_{j} |A_j|^2}{ 128 \sigma^2 \max_{j} |A_j|^2 } \big) ,
\EEAS
because the second term is always greater than the third one.
\item \textbf{Two-dimensional total variation}: we simply build the following counter-example:

\begin{center}\includegraphics[scale=.35]{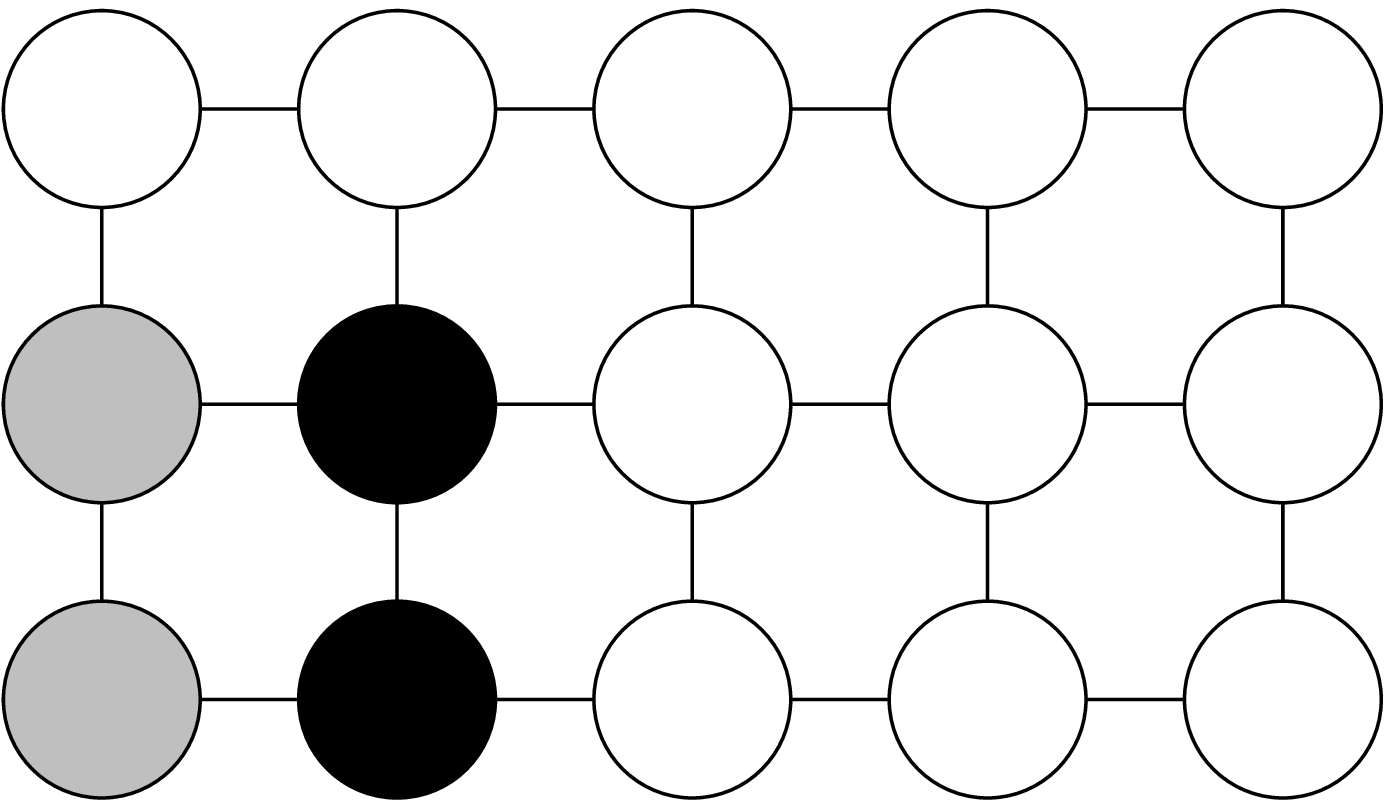}
\end{center}

where $B$ are the black nodes, $C$ the gray nodes and $A$ the complement of $B$. We indeed have $A$ connected, and $F(B \cup C) - F(B) = 4 - 5 = - 1$, $F(B \cup A) - F(B) = - 5$, leading to
$F(B \cup C) - F(B) - \frac{|C|}{|A|}[ F( B \cup A) - F(B) ] = -1 + 5 \times \frac{2}{13} = - \frac{3}{13}$.

We also illustrate this in \myfig{2dTV}, where we show that depending on the shape of the level sets (which still have to be connected), we may not recover the correct pattern, even with very small noise.

 \begin{figure}
\begin{center}
 
 \hspace*{-1.5cm}
\includegraphics[scale=.36]{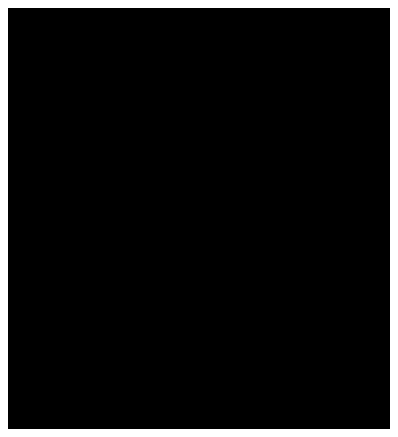} \hspace*{-1cm}
\includegraphics[scale=.36]{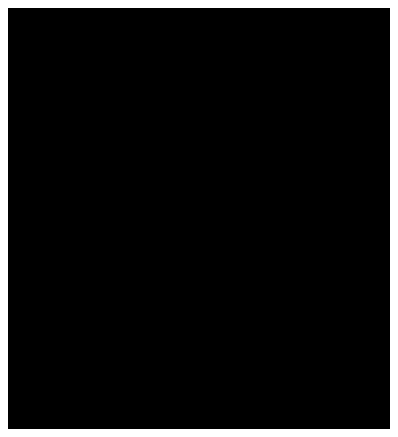} \hspace*{.1cm}
\includegraphics[scale=.36]{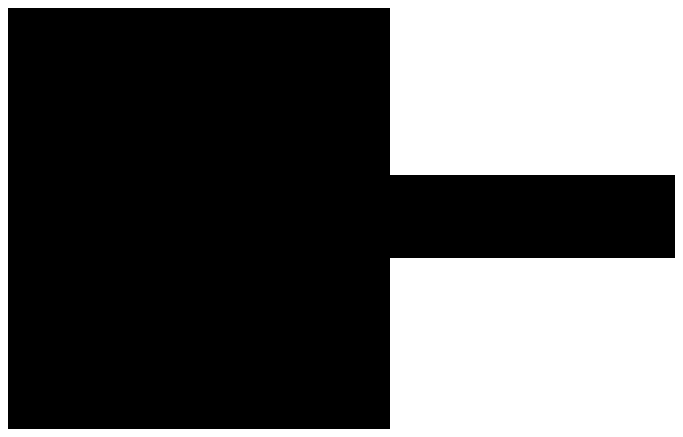}\hspace*{-1cm}
\includegraphics[scale=.36]{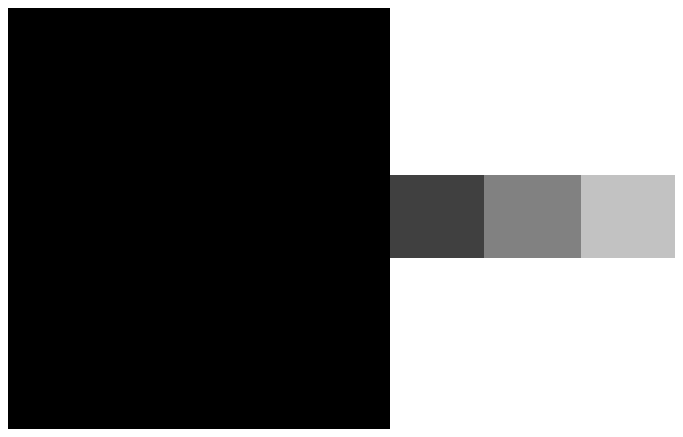}\hspace*{-1.5cm}

\vspace*{-1cm}

\end{center}
\caption{Signal approximation with the two-dimensional total variation: For two piecewise constant images with two values, the estimation may (left case) or may not (right case) recover the correct level sets, even with infinitesimal noise. For the two cases, left: original pattern, right: best possible recovered level sets.
}
\label{fig:2dTV}

\end{figure}

\EIT

 \bibliographystyle{unsrt}
\bibliography{submodular}

\end{document}